\documentclass[a4paper,10pt]{article}
\usepackage[utf8]{inputenc}
\usepackage{amsthm}
\usepackage{amsmath}
\usepackage{authblk}
\usepackage{mathtools}
\usepackage[mathscr]{euscript}
\usepackage{natbib}
 \bibpunct[, ]{[}{]}{,}{n}{}{,}%

\usepackage{amssymb}
\usepackage{amstext}
\usepackage{amsmath}
\usepackage{amsthm}
\usepackage{mathtools}
\usepackage{pslatex}
\usepackage{dsfont}
\usepackage{nicefrac}
\usepackage{float}
\usepackage{bm}
\usepackage{dsfont}
\def\url#1{\expandafter\string\csname #1\endcsname}
\newcommand{\gradg}[1]{\nabla_\theta \ell \left( #1 \right)}
\newcommand{\grad}{\nabla_\theta}

\usepackage{bm}
\newcommand{\W}{\bm{W}}
\newcommand{\Prob}[1]{P\left( #1 \right)}
\newcommand{\M}{\bm{\Psi}}
\newcommand{\Mw}{\bm{\Psi}^{\W \ub}}
\newcommand{\Mv}{\bm{\Psi}^{\vb b}}
\newcommand{\Dr}{\mathcal{D}_r}
\newcommand{\Dc}{\mathcal{D}_c}
\newcommand{\Xit}{\mathcal{X}}
\newcommand{\Yit}{\mathcal{Y}}
\newcommand{\ub}{\bm{u}}
\newcommand{\bb}{\bm{b}}
\newcommand{\x}{\bm{x}}
\newcommand{\xb}{\bm{x}}
\newcommand{\vb}{\bm{v}}
\newcommand{\Vb}{\bm{V}}
\newcommand{\nnw}{\bm{\theta}}
\newcommand{\sigmab}{\bm{\sigma}}
\newcommand{\sigmag}{\bm{\sigma}^g}
\newcommand{\sigmaip}{\bm{\sigma}^{ip}}
\newcommand{\sigmaop}{\bm{\sigma}^{op}}

\newcommand{\Fw}{\mathcal{F}^{\W \ub}}
\newcommand{\norm}[1]{\left\lVert #1 \right\rVert}
\newcommand{\abs}[1]{\left\lvert #1 \right\rvert}
\newcommand{\thetawu}{\theta^{\W \ub}}
\newcommand{\thetavb}{\theta^{\vb b}}
\newcommand{\gradvb}[1]{\nabla_{vb}\ell\left(#1 \right)}
\newcommand{\gradwu}[1]{\nabla_{wu}\ell\left(#1 \right)}
\newcommand{\derive}[2]{\frac{\partial #1}{\partial #2}}
\newcommand{\E}[1]{\mathbb{E}\left[ #1 \right]}
\newcommand{\Exp}[1]{\underset{(\x, y)\sim \mu}{\mathbb{E}}\left[ #1 \right]}
\newcommand{\angles}[1]{\left\langle #1 \right\rangle}
\newcommand{\set}[1]{\left\{ #1 \right\}}
\newcommand{\ind}[1]{\mathds{1}\left( #1 \right)}
\newcommand{\Mvb}{M^{\vb b}}
\newcommand{\Mwu}{M^{\W \ub}}
\newtheorem{lem}{Lemma}
\newtheorem{thm}{Theorem}
\newtheorem{A}{Assumption}
\newtheorem{fact}{Fact}
\title{A Framework for Provably Stable and Consistent Training of Deep Feedforward Networks}

\author{Arunselvan Ramaswamy \\
Department of Mathematics and Computer Science, Karlstad University, 651 88 Karlstad, Sweden, arunselvan.ramaswamy@kau.se \\
Shalabh Bhatnagar \\
Department of Computer Science and Automation, Indian Institute of Science, Bengaluru, India, shalabh@iisc.ac.in \\
Naman Saxena \\
Department of Computer Science and Automation, Indian Institute of Science, Bengaluru, India, namansaxena@iisc.ac.in
}
\begin{document}




\maketitle
\abstract{%
We present a novel algorithm for training deep neural networks in supervised (classification and regression) and unsupervised (reinforcement learning) scenarios. This algorithm combines the standard stochastic gradient descent and the gradient clipping method. The output layer is updated using clipped gradients, the rest of the neural network is updated using standard gradients. Updating the output layer using clipped gradient stabilizes it. We show that the remaining layers are automatically stabilized provided the neural network is only composed of squashing (compact range) activations. We also present a novel squashing activation function - it is obtained by modifying a Gaussian Error Linear Unit (GELU) to have compact range - we call it Truncated GELU (tGELU). Unlike other squashing activations, such as sigmoid, the range of tGELU can be explicitly specified. As a consequence, the problem of vanishing gradients that arise due to a small range, e.g., in the case of a sigmoid activation, is eliminated. We prove that a NN composed of squashing activations (tGELU, sigmoid, etc.), when updated using the algorithm presented herein, is numerically stable and has consistent performance (low variance). The theory is supported by extensive experiments. Within reinforcement learning, as a consequence of our study, we show that target networks in Deep Q-Learning can be omitted, greatly speeding up learning and alleviating memory requirements. Cross-entropy based classification algorithms that suffer from high variance issues are more consistent when trained using our framework. One symptom of numerical instability in training is the high variance of the neural network update values. We show, in theory and through experiments, that our algorithm updates have low variance, and the training loss reduces in a smooth manner.
}%

MSCClass: Primary: 90B05; secondary: 90C40, 90C90


\maketitle

%

\section{Introduction} \label{sec:intro}

The popularity iof Deep Neural Networks (DNNs) for solving a wide variety of supervised and unsupervised learning problems can be traced back to three milestones. First, the development of a neural network architecture based on the human eye for visual imagery analysis, called the Convolutional Neural Network (CNN). Its wide applicability for image classification and analysis was due to Lecun et. al. \cite{lecun1989backpropagation}, who showed that CNNs can be easily trained using labeled images and the backpropagation algorithm. The popularity of CNNs again got a major boost when Microsoft won the ImageNet contest in 2015 using their 100 layer CNN \cite{he2016deep}. The ImageNet is a hugely popular and important largescale visual object recognition software competition. The second milestone has been the development of Deep Q-Learning (DQL) by the researchers of DeepMind. DQL is a reinforcement learning algorithm that trains a DNN to aid in taking optimal actions in complex sequential decision making tasks. DeepMind popularly used DQL to autonomously learn the videogames ATARI \cite{2} and the boardgame GO \cite{silver2017mastering}. The algorithm was so effective that it beat the best human players in these games. The final milestone has been the quantum leaps in computational capability, specifically the development of powerful Graphics processing units (GPUs) - originally developed to play videogames. Deep Learning got a major boost from GPUs, since many algorithms particularly involving CNNs could be greatly sped up when deployed on them \cite{chellapilla2006high}.

Nowadays, DNNs are trained to recognize and respond to speech, to translate and summarize text, etc. Recently, ChatGPT, a chatbot based on GPT-4, developed by OpenAI has taken the world by storm \cite{gpt}. GPT-4 is a multimodal large language model that summarizes, generates and predicts new content using DNNs trained on massive data. Although applications of Deep Learning to real-world problems have become ubiquitous, we still do not know why, when and how they are effective \cite{sejnowski2020unreasonable}. Further, there are several unexplained paradoxes when training DNNs, greatly affecting credibility. The Learning community has started focusing on the mathematical theory of Deep Learning \cite{roberts2022principles}. As it is complicated to analyze autonomous decision making algorithms, developing the requisit mathematical framework is particularly hard for Deep Reinforcement Learning (DeepRL). The field of DeepRL involves training a DNN using the principles of dynamic programming. The most popular algorithm is called Deep Q-Learning (DQL), the DNN it trains is called Deep Q-Network (DQN). There have been some recent theoretical studies related to Deep Q-Learning, see for example, \cite{fan2020theoretical, fu2019diagnosing, sun2022finite, bowen2021finite}.

Theoretical studies in Deep Learning often develop requirements that guarantee stable training with asymptotic optimality. In practice, most of these requirements are hard to check, hence left unverified. Therefore, learning is not always stable, and when stable is not always optimal. Further, in practice, the algorithm performance is highly variable. It depends on parameters such as the initial random seed \cite{madhyastha2019model, qian2021my, pham2020problems}. This lack of consistency affects the trustworthiness of deep learning algorithms. In this paper, we take the first steps to fill the gap with respect to numerical stability and performance variability. We present a broad study covering both unsupervised and supervised learning.

Broadly speaking, in Deep Learning, a DNN is trained to minimize a given loss function with the help of data. In this paper, we consider (a) regression algorithms with the mean squared loss, (b) classification algorithms with the cross-entropy loss, and (c) Deep Q-Learning (DQL) with the squared Bellman loss. The reader is referred to \cite{goodfellow2016deep, bertsekas2019reinforcement} for more examples of loss functions. While our analysis easily covers many more loss functions, we restrict ourselves to these in order to keep the manuscript concise. With respect to DNN architectures, our analysis covers feedforward networks involving twice continuously differentiable squashing activations. By squashing, we mean that the function range is bounded and compact, despite the domain being potentially unbounded. We show, in theory and through experiments, that squashing activations contribute greatly to stability. We also address the main shortcoming of squashing activation functions - vanishing gradients - that arises due to the limited range of hitherto available squashing activations. For this, we present a novel squashing activation with extended range that we call \textit{Truncated Gaussian Error Linear Unit (tGELU)}. Training a DNN is a minimization process that is iterative in nature. Starting from a random initial value of the DNN weight vector, the aim is to iteratively modify it to find one that minimizes the loss function at hand. Stochastic Gradient Descent (SGD) is a popular algorithm for this process as it is simple yet effective. In this paper, we consider the SGD optimizer.

\subsection{Our Contributions and Relevance to Literature} \label{sec:contributions}

When using Stochastic Gradient Descent for training, Gradient clipping is a simple technique that ensures stability. The gradient value is scaled at every step before updating the neural network weight vector. The scaling is done to ensure that the gradient norm strictly stays below a predetermined ``clipping constant'' \cite{menon2020can, zhang2019gradient}. Using such bounded updates ensures numerical stability. Gradient clipping is especially popular in training large image classification networks that have a recurrent architecture, using the cross-entropy based loss function. 

Particularly in the field of Deep Reinforcement Learning, where the learning environment, and hence the training data, changes rapidly over time, stable training is a major challenge. Within Deep Q-Learning, a target network is an effective gadget to mitigate the stability issue. It is essentially a copy of the Deep Q-Network. While DQN is updated at every timestep, the target network is updated every $T >> 1$ ($T$ much greater than $1$) timesteps. Essentially, the DQN and target networks are synchronized every $T$ timesteps. Numerical instability in DQL occurs due to the high variance of the target value used to calculate the squared Bellman loss. A target network during training is typically used to calculate the target value, and since it is not frequently updated, it reduces variance of the target value and stabilizes training. The main downsides of using a target network are (a) high memory demand, (b) large overhead in maintaining and updating it, and (c) training is slow and the data needed to train DQN is significantly high. It must, however, be noted that using a target network does not guarantee stability but only raises it chances.

\textbf{Stability:} We propose building a DQN using twice continuously differentiable squashing activation functions. For training, we propose using the following modification of the standard SGD. At every timestep, update the output layer weights using clipped gradients, and the remaining weights (input and hidden layer) using regular sample gradients. \textit{We show that a DQN trained in this manner is numerically stable with probability $1$.} We also show that this statement holds for regression and classification algorithms. One significant contribution with respect to DQL is that \textit{we no longer require a target network, our framework stabilizes DQN training without one}. Additionally, in our experiments, we found that performance was better when using our framework, as opposed to traditional DQL using a target network. Eliminating target network is very useful since it frees up memory and speeds up training. Previous works such as \cite{kim2019deepmellow} also endeavor to eliminate target networks, but lack the mathematical backing, to do so. 

\textbf{Consistent Performance:} One major symptom of numerical instability is the high variance of the neural network weights, over the training duration. As stated earlier, parameters such as the initial random seed affect stability \cite{madhyastha2019model}. Scientifically speaking, this should not be the case. Also, high variance is directly linked to performance inconsistency. We show, in theory and through experiments, that DNNs trained using our framework have very low variance. Additionally, \textit{we show that the norm of the DNN weight vector is ``every moment bounded'' over the entire duration of training}. In particular, we show that our framework leads to consistent performance, \textit{independent of the choice of random seeds}.

\textbf{Truncated Gaussian Error Linear Unit (tGELU):} Our framework mandates that the DNN be composed of twice continuously differentiable squashing activations. If we consider the example of the sigmoid activation, its range is $[0,1]$, its derivative is approximately zero outside of the $[-4,4]$ interval. This leads to the vanishing gradients problem that renders SGD ineffective for training \cite{roodschild2020new}. For this reason, practitioners prefer activations with unbounded range such as the Gaussian Error Linear Unit (GELU). Its range is $[0, \infty)$ and, more importantly, its derivative is \textit{non-zero} on $[0, \infty).$ To overcome the vanishing gradients issue, we develop a novel (squashing) activation that we call Truncated Gaussian Error Linear Unit (tGELU). Its range is approximately $[t_l, t_r]$ and its  derivative is \textit{non-zero} on $[t_l, t_r],$ where $-\infty < t_l \le 0 < t_r < \infty.$ The parameters $t_l$ and $t_r$ are fixed by the experimenter. \textit{Our experiments suggest that tGELU facilitates stable training, leads to consistent and better performance, as compared to GELU or (Rectified Error Linear Unit) RELU activations.}

\textbf{DNN Initialization:} Finally, through our analysis we observe that the DNNs must be initialized such that the norm is upper-bounded by a function of the gradient clipping constant. We believe that the distribution used for initialization must have compact support, e.g., the truncated versions of the Gaussian or Laplace distributions. The framework and analysis is general, and can be modified to accommodate other types of training routines such as ADAM, RMSprop, etc.

\section{Definitions and notations} \label{sec:definitions}

In this section, we discuss the relevant notations and definitions.

We use bold-face capital letters to represent \textbf{matrices}, e.g., $\bm{W}, \bm{V}.$ The $i^{th}$ row of matrix $\bm{W}$ is represented by $\bm{W}_{i:}$, while its $j^{th}$ column is represented by $\bm{W}_{:j},$ the element that is in the $i^{th}$ row and $j^{th}$ column is represented by $\bm{W}_{i,j}.$ For \textbf{vectors}, we shall use bold-face small letters, e.g., $\bm{u}, \bm{v}.$ The $i^{th}$ component of $\bm{u}$ is represented by $\bm{u}_i.$ Unless otherwise stated, we shall reserve the sub-script notation to denote components of vectors and Matrices. We shall use $\nnw$ to represent the \textbf{vector of neural network weights}, the updated version at time step $n$ is given by $\nnw(n).$ All time indices are written in this manner, within rounded brackets.

First, recall that the filtration $\{\mathcal{F}(n)\}_{n \ge 0}$ is an increasing sequence of sigma-algebras, i.e., $\mathcal{F}(n) \subseteq \mathcal{F}(n+1), \ n \ge 0.$
A sequence or random variables $\{X(n)\}_{n \ge 0}$ is called a \textbf{martingale} sequence with respect to a filtration $\{\mathcal{F}(n)\}_{n \ge 0},$ when (a) $X(n)$ is $\mathcal{F}(n)$-measurable for all $n \ge 0$, (b) $X(n), \ n \ge 0,$ are integrable (c) $\E{X(n+1) \mid \mathcal{F}(n)} = X(n)$ a.s., for $n \ge 0.$ The random variable sequence constitutes a \textbf{submartingale}, when properties (a) and (b) hold, and $\E{X(n+1) \mid \mathcal{F}(n)} \ge X(n)$ a.s. The reader is referred to \cite{durrett2019probability} for more details.

Suppose the submartingale $\{X(n)\}_{n \ge 0}$ is such that $\abs{X(n) - X(n-1)} \le c(n),$ $c(n) < \infty$ and $n \ge 1.$ Then, for any $0 < N < \infty$ and $0 < \epsilon < \infty,$ the \textbf{Hoeffding-Azuma inequality} gives that:
\begin{equation}
    \label{eq:hoeffding}
    \Prob{\abs{X(N) - X(0)} \ge \epsilon} \le 2 \exp \left(\frac{- \epsilon^2}{2 \sum \limits_{n=1}^{N} c(n)^2} \right).
\end{equation}

Given a sequence of real numbers $\{a(n)\}_{n \ge 0},$ we write $a(n) \downarrow a,$ when $\lim \limits_{n \to \infty} a(n) = a$ and $a(n) > a(n+1)$ for $n \ge 0.$ Similarly, we write $a(n) \uparrow a,$ when $\lim \limits_{n \to \infty} a(n) = a$ and $a(n) < a(n+1)$ for $n \ge 0.$ 

Given a sequence of positive random variables $\{X(n)\}_{n \ge 0}$ ($X(n) \ge 0$ a.s., $n \ge 0$) such that $X(n) \uparrow X$ a.s. \textbf{Monotone Convergence Theorem} states that $\lim \limits_{n \to \infty}\E{X(n)} = \E{X}.$ When $\{X(n)\}_{n \ge 0}$ is a general sequence of random variables (takes positive and negative values with non-zero probability), such that $X(n) \uparrow X$ a.s., we still get that $\lim \limits_{n \to \infty}\E{X(n)} = \E{X}.$ The \textbf{Dominated Convergence Theorem} says that $\lim \limits_{n \to \infty}\E{X(n)} = \E{X},$ provided there exists some positive random variable $Y>0 \ a.s.$ such that $\E{Y} < \infty$ and $X(n) \le Y$ a.s. for $n \ge 0.$ In case, the random variable $Y$ is a constant - $Y \equiv C$ a.s. with $0 < C < \infty,$ then, this special case is called the \textbf{Bounded Convergence Theorem}. The reader is referred to \cite{rudin1976principles, durrett2019probability} for more details.

A point-to-set map $H: \mathbb{R}^m \to \{\text{subsets of }\mathbb{R}^n \},$ for some $m,n \ge 1,$ is called a \textbf{Marchaud map} when it possesses the following properties \cite{aubin2012differential}:
\begin{enumerate}
    \item $H(x)$ is convex and compact for every $x \in \mathbb{R}^m .$
    \item $\sup \limits_{z \in H(x)} \norm{z} \le K(1 + \norm{x})$ for some $K < \infty,$ $x \in \mathbb{R}^m .$
    \item Let $\lim \limits_{k \to \infty} z(k) = z$ in $\mathbb{R}^n,$ $\lim \limits_{k \to \infty} x(k) = x$ in $\mathbb{R}^m,$ with $z(k) \in H(x(k))$ for all $k.$ Then, $z \in H(x).$ This property is called \textbf{upper semicontinuity.}
\end{enumerate}

\section{Neural network architectures and the loss derivatives}\label{sec:nn}
In this section, we introduce the neural network architectures and the associated notations that are needed to present our analyses. Although our analyses are equally applicable for general fully connected deep feedforward networks, we only describe shallow networks here. We do this in order to utilize simple easy-to-remember notations and to present clear succinct analyses. When appropriate, we briefly describe the extensions needed to accommodate deep networks as well. We consider the gamut of deep learning problems, so we will describe generic network architectures for regression, classification and reinforcement learning. We make the following important assumption with respect to the activation functions used to construct our neural networks.

\begin{A}
    \label{asmp:nnact}
    The neural networks are constructed with at least twice continuously differentiable squashing activation functions. In particular, they have bounded derivatives. Examples include sigmoid, tanh and Gaussian activations.
\end{A}
\subsection{Regression} 
\begin{figure}[h]
    \centering
    \includegraphics{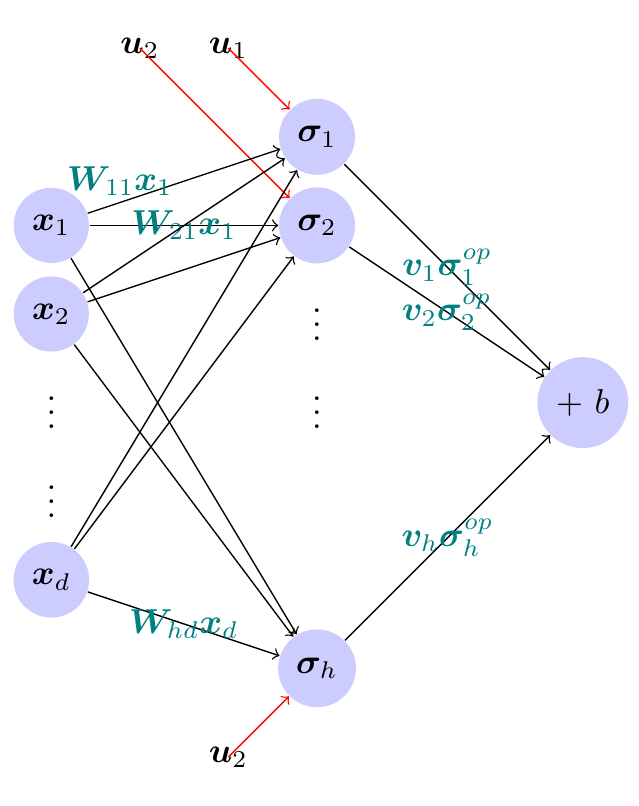}
    \caption{A shallow feedforward neural network for regression}
    \label{fig:ffn}
\end{figure}
Fig.~\ref{fig:ffn} illustrates a shallow fully connected feedforward neural network (NN) for regression problems. Such a NN is typically trained using a dataset $\Dr \equiv \{(x(n), y(n)) \mid x(n) \in \Xit,$ $y(n) \in \Yit, 1 \le n \le N\},$ where $\Xit \subseteq \mathbb{R}^d$ and $\Yit \subset \mathbb{R}.$ In words, $\Xit$, the input space is a subset of some $d$ dimensional real-space and the output space, $\Yit$, is a subset of the real-space. The objective is to train a NN to minimize a loss function such as the mean-squared-loss.

\begin{A}
    \label{asmp:XYcompact}
    $\Xit$ and $\Yit$ are compact subsets of $\mathbb{R}^d$ and $\mathbb{R},$ respectively.
\end{A}

In Fig.~\ref{fig:ffn}, $\x \equiv (\x_1, \ldots, \x_d)$ is the $d$-dimensional input to the NN with $1$ hidden layer with $h$ activations, $\W$ is a $h \times d$ matrix, $\ub$ and $\vb$ are $h$ dimensional vectors, and $b$ is a scalar. These constitute the neural network weights $\nnw \equiv (\W, \ub, \vb, b).$ Since the neural network weights are collated together in a vector form, $\nnw$ is referred to as the weight vector. Further, we let $\thetavb \equiv (\vb, b)$ and $\thetawu \equiv (\W, \ub),$ hence $\nnw \equiv (\thetavb, \thetawu).$ Let $\sigmab_1, \ldots, \sigmab_h$ represent the $h$ activation functions. Assumption~\ref{asmp:nnact} restricts our choice of activation functions to squashing activations that are at least twice continuously differentiable. Examples of such activations include sigmoid, tanh, etc. In Section~\ref{sec:exp}, we present novel modifications to current activation functions that satisfy Assumption~\ref{asmp:nnact}. We show that via such modifications the NNs are able to achieve better performance in terms of both numerical stability and optimality. For example, in Fig.~\ref{fig:tgelu} we illustrate the modifications to GeLU and Tanh activations. In Fig.~\ref{fig:tgeluvsgelu}, we show that these modified activations, which satisfy Assumption~\ref{asmp:nnact}, result in a better performing NN for supervised learning tasks.

The input to the $k^{th}$ activation is $\sigmaip_k \equiv \angles{\W_{k:},\x} + \ub_k,$ where $\W_{k:}$ represents the $k^{th}$ row of the $\W$ matrix. Its output is $\sigmaop \equiv \frac{1}{1 + \exp{(- \sigmaip_k)}}$ when the activation function is sigmoid. Let us define $\sigmab \equiv \left( \sigmaop_1, \ldots, \sigmaop_h \right).$ Then, the output of the NN is $f(\xb, \nnw) \equiv \angles{\sigmab, \vb} + b.$ When processing datapoint $(x,y)$ to train the NN, the mean squared error is given by $\ell_r(\nnw(n), \x(n), y(n)) \equiv \left[ f(\xb(n), \nnw(n)) - y(n) \right]^2.$ Stochastic gradient descent to update the NN weight vector is given by
\begin{equation} \label{eq:regressionupdate_old}
    \nnw(n+1) = \nnw(n) - a(n) \grad \ell_r(\nnw(n), \x(n), y(n)),
\end{equation}
where $\set{a(n)}_{n \ge 0}$ is the given step-size sequence or learning rate. Typically, at time-step $n$ multiple datapoints are processed in order to update the weight vector. In equation \eqref{eq:regressionupdate_old} we consider single-point updates since the analyses for single-point and batch updates are identical, but the kitsch is greatly reduced in the former. The components of $\grad \ell_r(\nnw, x, y)$ are listed below:
\begin{equation}
\label{eq:reggrad}
\begin{split}
    \derive{\ell_r(\nnw, \x, y)}{b} &= 2 (f(\x, \nnw) - y), \\
    \derive{\ell_r(\nnw, \x, y)}{\vb_k} &= 2 (f(\x, \nnw) - y) \sigmaop_k,\ 1 \le k \le h, \\
    \derive{\ell_r(\nnw, \x, y)}{\ub_k} &= 2 (f(\x, \nnw) - y)\sigmag_k \vb_k, \ 1 \le k \le h, \\
    \derive{\ell_r(\nnw, \x, y)}{\W_{jk}} &= 2 (f(\x, \nnw) - y)\sigmag_k \vb_k \x_j, 1 \le k \le h \text{ and } 1 \le j \le d.
    \end{split}
\end{equation}

In the above, $\sigmag_k$ represents the derivative of the $k^{th}$ activation function evaluated at the input. E.g., when the activation function is a sigmoid, then $\sigmag_k = \sigmaop_k (1 - \sigmaop_k)$.
We split the components of $\grad \ell_r$ into two, viz. $\nabla_{vb}\ell_r$ and $\nabla_{wu}\ell_r,$ such that we define $\nabla_{vb}\ell_r \equiv \left( \derive{\ell_r}{\vb_1}, \ldots \derive{\ell_r}{\vb_h}, \derive{\ell_r}{b} \right)$ and $\nabla_{wu}\ell_r \equiv \left( \derive{\ell_r}{\W_{11}}, \ldots \derive{\ell_r}{\W_{dh}}, \derive{\ell_r}{\ub_1}, \ldots, \derive{\ell_r}{\ub_h} \right).$

\begin{lem}
    \label{lem:wubound}
    Under Assumptions~\ref{asmp:nnact} and \ref{asmp:XYcompact}, (i) $\abs{f(\x, \nnw)} \le K(1 + \norm{\thetavb}),$ (ii) $\norm{\nabla_{vb} \ell_r(\nnw, \x, y} \le K \left(1 + \norm{\thetavb} \right)$ and (iii) $\norm{\nabla_{wu} \ell_r(\nnw, \x, y} \le K \left(1 + \norm{\thetavb}^2 \right),$ for some $0 < K <  \infty;$ $\nnw$ is the neural network weight vector; $\x \in \Xit$ and $y \in \Yit.$
\end{lem}

\begin{proof}

\vspace{.2cm}

In order to show (i), we observe that $\abs{f(\x, \nnw)} \le \abs{\angles{\sigmab, \vb}} + \abs{b}.$ Using the Cauchy-Schwarz inequality, we get $\abs{\angles{\sigmab, \vb}} \le \norm{\vb} \norm{\sigmab}.$ Due to Assumption~\ref{asmp:nnact}, we get that $\norm{\sigmab} \le K_1$ for some $K_1 < \infty.$ Finally, since $\abs{b} \vee \norm{\vb} \le \norm{\thetavb},$ we get that there exists $K < \infty$ such that
\begin{equation}
    \label{wubound1}
    \abs{f(\x, \nnw)} \le K(1 +  \norm{\thetavb}).
\end{equation}

Assumption~\ref{asmp:XYcompact} tells us that we also have

\begin{equation}
    \label{wubound2}
    \sup \limits_{x \in \Xit, y \in \Yit} \ 2\abs{f(\x, \nnw) - y} \le K(1 +  \norm{\thetavb}),
\end{equation}
where, without loss of generality, the constant $K$ is from \eqref{wubound1} as we can choose $K$ to be the maximum of the two constants, otherwise. This directly yields

\begin{equation}
    \label{wubound3}
    \norm{\gradvb{\nnw, \x, y}} \le K \left(1 + \norm{\thetavb} \right).
\end{equation}

We are left to prove (iii). First, observe that, $\abs{\sigmag_k} \le K_2$ and $\abs{\x_j} \le K_3$ for some $K_2, \ K_3 < \infty,$ as a consequence of Assumptions~\ref{asmp:nnact} and \ref{asmp:XYcompact} respectively, $1 \le k \le h$ and $1 \le j \le d.$ Using these, and previously made observations, we get that $\abs{\derive{\ell_r(\nnw, \x, y)}{\W_{jk}}} \le 2 \abs{(f(\x, \nnw) - y)\sigmag_k \vb_k \x_j} \le K_4 (1 + \norm{\thetavb}^2),$ and that $\abs{\derive{\ell_r(\nnw, \x, y)}{\ub_k}} \le K_4 (1 + \norm{\thetavb}^2),$ for some $K_4 < \infty,$ $1 \le k \le h.$ Hence,

\begin{equation}
    \label{wubound4}
    \norm{\gradwu{\nnw, \x, y}} \le K \left(1 + \norm{\thetavb}^2 \right).
\end{equation}

Again, we may assume that the constant $K$ remains unchanged.

\end{proof}

In this paper, we show that the NN can be trained in a numerically stable manner, merely by ensuring numerical stability of the output layer. In particular, we consider and analyze the following update sequence:

\begin{equation} \label{eq:regressionupdate}
    \begin{split}
        \thetavb(n+1) &= \thetavb(n) - a(n) \frac{\nabla_{vb} \ell_r(\nnw(n), \x(n), y(n))}{\large\nicefrac{\norm{\nabla_{vb} \ell_r(\nnw(n), \x(n), y(n))}}{\lambda} \vee 1}, \\
        \thetawu(n+1) &= \thetawu(n) - a(n) \nabla_{wu}\ell_r(\nnw(n), \x(n), y(n)),
    \end{split}
\end{equation}

where $\lambda < \infty$ is a predetermined ``clipping constant''. $\thetavb$ is updated using norm-based gradient clipping, and the update (loss-gradient value) is norm-bounded at every step by $\lambda.$ While the output layer is updated using the clipping method, the inner layer weight vector $\thetawu$ is updated using the standard gradient descent method.

%
%
%

\subsection{Classification} \label{sec:crossentropy}

Similar to regression, every classification problem begins with a dataset $\Dc \equiv \set{(x(n), y(n)) \ \lvert 1 \le n \le N},$ where $x(n) \in \mathbb{R}^d$ and $y(n) \in \{0, \ldots, k-1\}.$ The output instances are also called class labels, and $k$ is the number of classes.
Fig~\ref{fig:crossentropynn} illustrates a shallow feedforward network for the binary classification problem ($k = 2$). The output layer is called the softmax layer. The weights of this network $\theta \equiv \left(\W, \ub, \Vb, \bb \right),$ where $\W$ is a $h \times d$ matrix, $\ub$ is a vector of dimension $h$, $\Vb$ is a matrix of dimension $2 \times h,$ and $\bb$ is a vector of dimension $2.$ Let $\thetawu \equiv \left( \W, \ub \right)$ and $\thetavb \equiv \left( \Vb, \bb \right)$ The cross-entropy loss function is a popular choice to train such a NN. For the binary classification setting, it is given by:
\begin{equation}
    \label{eq:crossentropyloss}
   \ell_c(\theta, \x, y) \equiv -\left[ \mathds{1}(y =0) \log (z_1) + \mathds{1}(y=1) \log (z_2) \right],
\end{equation}
where $z_1$ and $z_2$ are described in Fig.~\ref{fig:crossentropynn}, and $y$ represents the true class label of the the input instance $x$ being processed to train the NN. First, we note that
\begin{figure}
    \centering
    \includegraphics{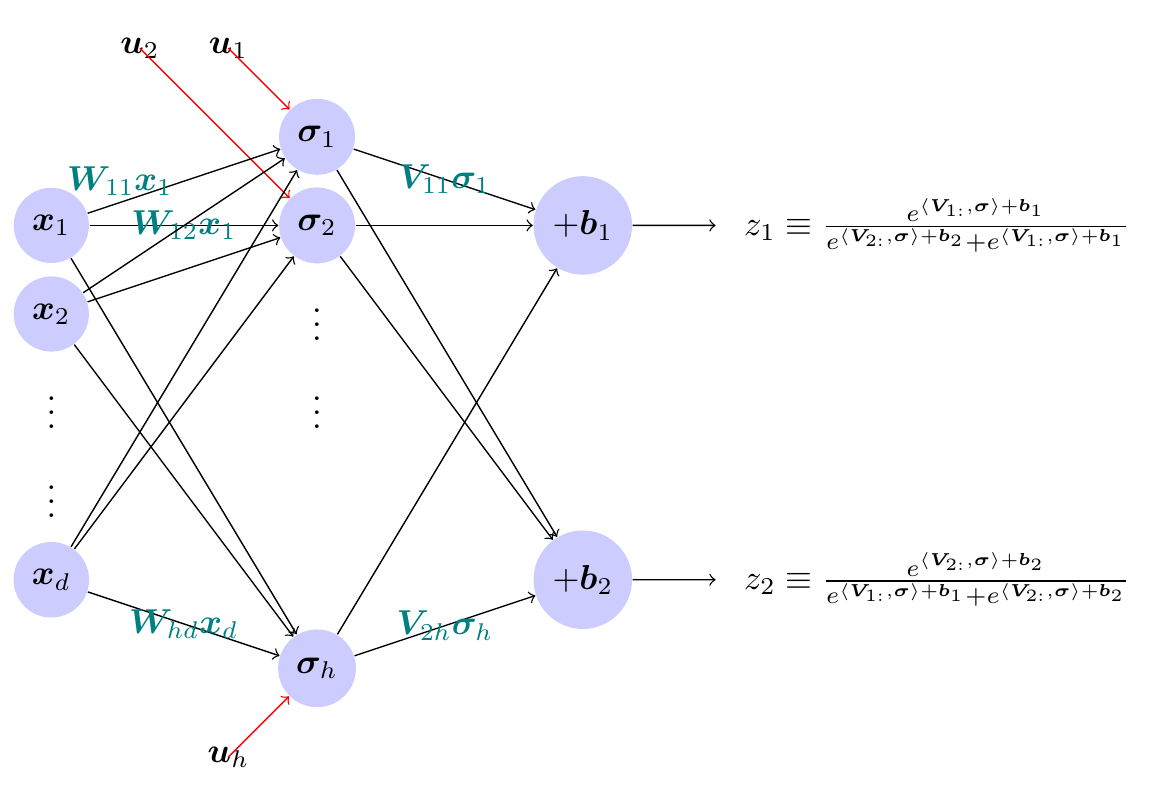}
    \caption{Feedforward network with a soft-max output layer}
    \label{fig:crossentropynn}
\end{figure}
\begin{equation}
    \label{eq:crossentropygradz}
    \begin{split}
        \derive{\ell_c(\theta, \x, y)}{z_1} &= -\frac{\left[ \mathds{1} (y=0) (1 - z_1) - \mathds{1}(y=1) z_1 \right]}{z_1(1-z_1)} = \frac{z_1 - \mathds{1} (y=0)}{z_1(1-z_1)}, \\
        \derive{\ell_c(\theta, \x, y)}{z_2} &= -\frac{\left[ -\mathds{1} (y=0) z_2 + \mathds{1}(y=1) (1-z_2) \right]}{z_2(1-z_2)}= \frac{z_2 - \mathds{1} (y=1)}{z_2(1-z_2)},\\
        \derive{\ell_c(\theta, \x, y)}{c_1} &= z_1 - z_2, \\
        \derive{\ell_c(\theta, \x, y)}{c_2} &= z_2 - z_1,
    \end{split}
\end{equation}
 where $c_i \equiv \angles{\Vb_{i:}, \sigmaop} + \bb_i,$ $1 \le i \le 2$. Note that $z_1(1-z_1) = z_2(1-z_2) = z_1z_2.$ Then, we can derive the following:

\begin{equation} \label{eq:crossentropygrad}
    \begin{split}
        \derive{\ell_c(\theta, \x, y)}{\Vb_{ij}} &=  \sigmab_j \left(z_i - z_{i'} \right), \\
        \derive{\ell_c(\theta, \x, y)}{\bb_i} &= \left(z_i - z_{i'} \right), \\
        \derive{\ell_c(\theta, \x, y)}{\W_{jk}} &= \x_k \derive{\sigmaop_j}{\sigmaip_j} \left[ \left(z_1 - z_{2} \right) \Vb_{1j} + \left(z_2 - z_{1} \right) \Vb_{2j} \right], \\
        \derive{\ell_c(\theta, \x, y)}{\ub_j} &= \derive{\sigmaop_j}{\sigmaip_j} \left[ \left(z_1 - z_{2} \right) \Vb_{1j} + \left(z_2 - z_{1} \right) \Vb_{2j} \right].
    \end{split}
\end{equation}
where $1 \le i \le 2,$ $i' = 3-i,$ $1 \le j \le h$ and $1 \le k \le d.$ Let $\nabla_{wu} \ell_c(\theta, \x, y)$ represent the vector of partial derivatives with respect to the $\W_{ij}'s$ and $\ub_i's$ and let $\nabla_{vb} \ell_c(\theta, \x, y)$ represent the vector of partial derivatives with respect to the $\Vb_{jk}'s$ and $\bb_j's.$ Now, we state something similar to the statement of Lemma~\ref{lem:wubound}.

\begin{lem}
    \label{lem:wubound2} Under Assumptions~\ref{asmp:nnact} and \ref{asmp:XYcompact},
    $\norm{\nabla_{vb} \ell_c(\theta, \x, y)} \le K$ and $\norm{\nabla_{wu} \ell_c(\theta, \x, y)} \le K(1 + \norm{\thetavb})$ for some $K < \infty,$ neural network weight vector $\theta$, $\x \in \Xit$ and $y \in \Yit.$
\end{lem}

\begin{proof}

\vspace{.2cm}

The proof proceeds along similar lines to that of Lemma~\ref{lem:wubound}. The required results are obtained by bounding the right hand sides of \eqref{eq:crossentropygradz} and \eqref{eq:crossentropygrad}.

\end{proof}

Since the gradient values used to update the output layer are bounded, we do not need to explicitly clip them. Alternatively, if we choose $\lambda \equiv K+1,$ then it follows from Lemma~\ref{lem:wubound2} that the clipping condition will never be satisfied. Hence, when using the cross entropy loss, the NN can be updated using the following simple gradient descent algorithm:

\begin{equation} \label{eq:classificationupdate}
    \nnw(n+1) = \nnw(n) - a(n) \nabla_\theta \ell_c(\nnw(n), \x(n), y(n)).
\end{equation}

\subsection{Deep Q-Learning} \label{sec:dqn}

Deep Q-Learning is an important algorithm for solving sequential decision making problems. The goal is to interact with an underlying system over time, and take a sequence of decisions that maximizes the accumulated rewards. At time $t,$ decision $a(t)$ causes the underlying system to transition from state $x(t)$ to state $x(t+1).$ The system, at every instant, provides feedback in the form of a real-valued reward $r(x(t), a(t)).$ Formally speaking, in Deep Q-Learning, the Q-Network, illustrated in Fig.~\ref{fig:dqnn}, is trained to find a policy $\pi: \mathcal{S} \to \mathcal{A}$ such that $\sum \limits_{s \ge t_0} \gamma^{s - t_0} r(x(s), \pi(x(s)))$ is maximized, $0 \le \gamma \le 1$ is the discount factor and $\pi$ is a mapping from the system state-space $\mathcal{S}$ to the decision space $\mathcal{A}.$ The Q-Network weights $\nnw \equiv (\W, \ub, \Vb, \bb)$ parameterize the policy space, and the policy associated with $\nnw$ is represented by $\pi(\cdotp; \nnw),$ such that $\pi(\x, \nnw) = \underset{a \in \mathcal{A}}{\text{argmax}} Q(x,a, \nnw).$ It additionally parameterizes the space of Q-factors. Specifically, the Q-factor $Q(\x(t_0), a(t_0), \theta) = r(x(t_0), a(t_0)) + \mathbb{E}_{\x(s) \sim \tau(s), s > t_0}\left[\sum \limits_{s \ge t_0 + 1} \gamma^{s - t_0} r(x(s), \pi(\x(s), \nnw))\right],$ where $\tau(\cdotp)$ represents the, possibly time-varying, state transition kernel. It represents the expected cumulative discounted reward obtained when action $a(t_0)$ is picked when the system is in $x(t_0),$ following which the actions are picked in accordance to $\pi(\cdotp, \theta).$ The expected squared Bellman loss that must be calculated at time $t$ to train the DQN is given by

\begin{equation}
    \label{eq:Bellmanloss_expected}
    \ell_b(\nnw, \x,a) \equiv \left[ Q(\x,a,\nnw) - r(\x,a) - \gamma \ \max \limits_{a' \in \mathcal{A}} \ \mathbb{E}_{\x' \sim \tau(t)} Q(\x', a', \nnw)  \right]^2,
\end{equation}

where the system is in state $\x$ at time $t,$ and some action $a$ is taken.
In practice however, the state transition kernel is unknown. Hence, in order to find $\theta^*$ such that $Q(\x, \pi(\x,\nnw^*)) \ge Q(\x, \hat{\pi}(\x)),$ for every other policy $\hat{\pi}$, we train the Q-Network to minimize the following noisy sample-based squared Bellman:

\begin{equation}
    \label{eq:Bellmanloss}
    \ell_b(\nnw, \x,a) \equiv \left( Q(\x,a,\nnw) - r(\x,a) - \gamma \ \max \limits_{a' \in \mathcal{A}} Q(\x', a', \nnw) \right)^2.
\end{equation}

\begin{figure}
    \centering
    \includegraphics{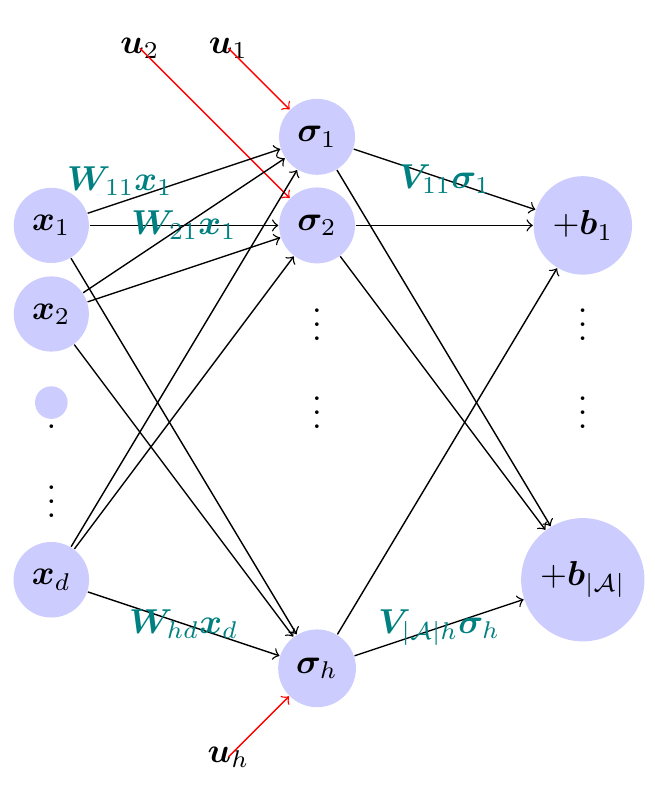}
    \caption{A shallow Q-Network}
    \label{fig:dqnn}
\end{figure}

The input to the Q-Network, illustrated in Fig.~\ref{fig:dqnn}, is the a state value $\x,$ in this case we have $\x \in \mathbb{R}^d.$ The Q-Network weight vector $\theta \equiv (\W, \ub, \Vb, \bb),$ where $\W$ is a $h\times d$ matrix, $\ub$ is a $h$ dimensional vector, $\Vb$ is a $\abs{\mathcal{A}} \times h$ matrix, and $\bb$ is a $\abs{\mathcal{A}}$ dimensional vector. Let $\mathcal{A} \equiv \{a_1 \ldots, a_M\},$ then $Q(\x, a_m, \theta) \equiv \angles{\sigmab, \Vb_{m:}} + \bb_m,$ $1 \le m \le M,$ $\sigmab \equiv (\sigmaop_1, \ldots, \sigmaop_h).$ As in the previous two sections, we list the relevant loss gradients below. Note that they are based on the sample Bellman loss. In particular, at the time of calculating the Bellman loos, the next state $\x'$ is sampled from the system when action $a$ is taken in state $\x.$
\begin{equation}
    \label{eq:dqngradient}
    \begin{split}
        \derive{\ell_b(\nnw, \x,a)}{\Vb_{ij}} &= 2 \left( Q(\x,a,\nnw) - r(\x,a) - \gamma \ \max \limits_{a' \in \mathcal{A}} Q(\x', a', \theta) \right) \sigmab_j, \\
        \derive{\ell_b(\nnw, \x,a)}{\bb_i} &= 2 \left( Q(\x,a,\nnw) - r(\x,a) - \gamma \ \max \limits_{a' \in \mathcal{A}} Q(\x', a', \nnw) \right), \\
        \derive{\ell_b(\nnw, \x,a,)}{\W_{jk}} &= 2 \left(  Q(x,a,\nnw) - r(\x,a) - \gamma \ \max \limits_{a' \in \mathcal{A}} Q(\x', a', \nnw) \right) \x_j \derive{\sigmaop_j}{\sigmaip_j} \sum \limits_{m=1}^{|\mathcal{A}|} \Vb_{mj}, \\
        \derive{\ell_b(\nnw, \x,a)}{\ub_j} &= 2 \left(  Q(\x,a,\nnw) - r(\x,a) - \gamma \ \max \limits_{a' \in \mathcal{A}} Q(\x', a', \nnw) \right) \derive{\sigmaop_j}{\sigmaip_j} \sum \limits_{m=1}^{|\mathcal{A}|} \Vb_{mj},
    \end{split}
\end{equation}

where $\mathcal{A}$ is the finite action space, $1 \le i \le |\mathcal{A}|,$ $1 \le j \le h$ and $1 \le m \le d.$ Let $\thetavb \equiv (\Vb, \bb),$ $\thetawu \equiv (\W, \ub),$ $\nabla_{vb} \ell_b(\nnw, \x,a) \equiv \left( \derive{ \ell_b(\nnw,\x,a)}{\Vb_{ij}}, \derive{\ell_b(\nnw,\x,a)}{\bb_i} \right)_{1 \le i \le |\mathcal{A}|, 1 \le j \le h}$ and $\nabla_{wu} \ell_b(\nnw,\x,a) \equiv \left(\derive{\ell_b(\nnw,\x,a)}{\W_{jk}}, \derive{\ell_b(\nnw,\x,a)}{\ub_j} \right)_{1 \le k \le d, 1 \le j \le h}.$ We make the following assumption with respect to the rewards.

\begin{A}
    \label{asmp:reward}
    $\mathcal{S} \subset \mathbb{R}^d$ is compact, $\mathcal{A}$ is finite and discrete, and $\sup \limits_{\x \in \mathcal{S}, \ a \in \mathcal{A}} \abs{r(\x, a)} < \infty.$
\end{A}

Recall that $Q(\x, a_m, \nnw) \equiv \angles{\sigmab, \Vb_{m:}} + \bb_m,$ $1 \le m \le M,$ where $M$ is the total number of possible actions, it now follows from Assumption~\ref{asmp:nnact} that $\max \limits_{a \in \mathcal{A}} \abs{Q(\x, a, \nnw)} \le K(1 + \norm{\thetavb})$ for some $K < \infty.$ Without loss of generality, we have the same $K$ as Lemmas~\ref{lem:wubound} and \ref{lem:wubound2}. Hence, we may say that $\left( Q(\x,a,\nnw) - r(\x,a) - \gamma \ \max \limits_{a' \in \mathcal{A}} Q(\x', a', \theta) \right)$ in \eqref{eq:dqngradient} is akin to $\left( f(\x, \theta) - y \right)$ in \eqref{eq:reggrad}. Then, the partial derivatives \eqref{eq:dqngradient} are similar to those listed in \eqref{eq:reggrad}. We can therefore expect them to have similar properties. We state the following Lemma without proving it, since the steps involved are identical to those from the proof of Lemma~\ref{lem:wubound}.

\begin{lem}
    \label{lem:wubound3}
    Under Assumptions~\ref{asmp:nnact}, \ref{asmp:XYcompact} and \ref{asmp:reward}, (i) $ \abs{ Q(\x, a, \nnw)} \le K(1 + \norm{\thetavb}),$ (ii) $\norm{\nabla_{vb} \ell_b(\nnw,\x, a)} \le K \left(1 + \norm{\thetavb} \right)$ and (iii) $\norm{\nabla_{wu} \ell_b(\nnw,\x, a)} \le K \left(1 + \norm{\thetavb}^2 \right),$ for some $0 < K <  \infty;$ $\nnw$ is the neural network weight vector; $\x \in \mathcal{S}$ and $a \in \mathcal{A}.$
\end{lem}

We thus analyze the following variant of the Deep Q-Learning algorithm where the output layer alone is updated using clipping.

\begin{equation} \label{eq:dqnupdate}
    \begin{split}
        \thetavb(n+1) &= \thetavb(n) - a(n) \frac{1}{K} \sum \limits_{k=1}^K \frac{\nabla_{vb} \ell_b(\nnw(n), \x(n,k), \mathds{\alpha}(n,k))}{\large\nicefrac{\norm{\nabla_{vb} \ell_b(\nnw(n), \x(n,k), \mathds{\alpha}(n,k))}}{\lambda} \vee 1}, \\
        \thetawu(n+1) &= \thetawu(n) - a(n) \frac{1}{K} \sum \limits_{k=1}^K \nabla_{wu}\ell_b(\nnw(n), \x(n,k), \mathds{\alpha} (n,k)),
    \end{split}
\end{equation}
where $K$ is the mini-batch size, and $(\x(n,k), \alpha (n,k)), \x'(n,k))$ is a sample from the experience replay buffer that stores past system interactions. Recall that the first component of the tuple is the system state, the second is the action taken in that state, and the third is the state to which the system transitions as a consequence.

\section{Algorithm, stability and moment bounds} \label{sec:algo}
In this section, we use the technical lemmas proven in the previous section to show the stability of the modified regression classification and Q-Learning updates, described by \eqref{eq:regressionupdate}, \eqref{eq:classificationupdate} and \eqref{eq:dqnupdate}, respectively. We will also show that the weight updates satisfy certain desirable moment conditions. In order to avoid redundancy, we will analyze the following generic iteration:
\begin{equation}
    \label{eq:genericupdate}
     \begin{split}
        \thetavb(n+1) &= \thetavb(n) - a(n) \frac{\nabla_{vb} \ell(\nnw(n), \x(n), y(n))}{\large\frac{\norm{\nabla_{vb} \ell(\nnw(n), \x(n), y(n))}}{\lambda} \vee 1} , \\
        \thetawu(n+1) &= \thetawu(n) - a(n) \nabla_{wu}\ell(\nnw(n), \x(n), y(n)),
    \end{split}
\end{equation}
where $\thetavb \equiv (\vb, \bb)$ and $\thetawu \equiv (\W, \ub)$ generically represent the output and the input layer weights, respectively. Note that we have dropped the subscript from the loss function $\ell$. Lemmas~\ref{lem:wubound}, \ref{lem:wubound2} and \ref{lem:wubound3} make qualitatively similar statements in the regression, classification and Q-Learning contexts, respectively. We will only need these Lemmas for the analyses presented in this section. Also, note that we have omitted the mini-batch updates. Again, it will be clear that the stability analysis in this section will remain unaltered in the presence of mini-batches for training. It will become relevant when analyzing the convergence properties, especially for Q-Learning. Again, the generic loss function, $\ell \equiv \ell_r$ for regression and Lemma~\ref{lem:wubound}, $\ell \equiv \ell_c$ for classification and Lemma~\ref{lem:wubound2}, and $\ell \equiv \ell_b$ for Q-Learning and Lemma~\ref{lem:wubound3} become relevant. 

\subsection{The output layer behavior}

We make an important assumption with regards to the initialization of the neural network weights. Although we allow for random initialization, we assume that they are norm-bounded by a prespecified value.
\begin{A} \label{asmp:initial}
$\norm{\theta(0)} \le \lambda \ a.s.,$ where $\lambda < \infty$ is a prespecified fixed value.
\end{A}
\begin{A} \label{asmp:steps}
$a(m) > 0$ for all $m \ge 0,$ $\sum \limits_{m \ge 0} a(m) = \infty$ and $\sum \limits_{m \ge 0} a(m)^2 < \infty.$
\end{A}
The output layer weights - $\thetavb$ - are updated using clipped gradients, the $(n+1)^{st}$ update step is given by
\begin{equation}
\label{eq:vbclip}
    \thetavb(n+1) = \thetavb(n) - a(n) g(n),
\end{equation}
where $g(n)$ is the clipped version of $\gradvb{\thetavb(n), x(n), y(n)},$ such that $\norm{g(n)} \le \lambda$ for some prespecified fixed $\infty > \lambda > 0.$ Suppose, we use the classical norm based clipping method, then $$g(n) \equiv \frac{\gradvb{\thetavb(n), x(n), y(n)}}{\frac{\norm{\gradvb{\thetavb(n), x(n), y(n)}}}{\lambda} \vee 1}.$$ 

Let us define a stochastic process $\{\Mv(n)\}_{n \ge 0}$ and associate the natural filtration with it: $\Mv(0) \equiv 0$ and $\Mv(n) = \norm{\thetavb(0)} + \sum \limits_{m=0}^{n-1} a(n) \norm{g(n)},$ $n \ge 1.$ It is easy to see that $\{\Mv(n)\}_{n \ge 0}$ is a submartingale. We further get that $\abs{\Mv(n) - \Mv(n-1)} \le a(n-1)\lambda$ for $n \ge 2,$ provided we initialize the weights such that their norm is less than $\lambda.$ From the Hoeffding-Azuma inequality, we get
\begin{equation}
    \label{eq:HA}
    P(\Mv(n) \ge x) \le \exp{\left( \frac{-x^2}{\lambda^2 \sum \limits_{m \ge 0} a(m)^2} \right)} \text{ for all } x \ge 0.
\end{equation}
\begin{lem} \label{lem:vbstable}
$\sup \limits_{m \ge 0} \norm{\thetavb(m)} < \infty \ a.s.$
\end{lem}
\begin{proof}

\vspace{.2cm}

Let us recall from above that $P(\Mv(n) \ge x) \le \exp{\left( \frac{-x^2}{\lambda^2 \sum \limits_{m \ge 0} a(m)^2} \right)}$ for all $x \ge 0, \ n \ge 0.$ From this we get, $\lim \limits_{x \to \infty} P(\Mv(n) \ge x) = P(\Mv(n) = \infty) = 0,$ and hence, $\Prob{\Mv(n) < \infty} = 1.$ Since $\Mv(n) \ge \sup \limits_{0 \le m \le n} \ \norm{\thetavb(n)} \ a.s.,$ we have $\Prob{\sup \limits_{0 \le m \le n} \ \norm{\thetavb(m)} < \infty} = 1.$ Now, the event sequence $\left\{\sup \limits_{0 \le m \le n} \ \norm{\thetavb(m)} < \infty \right\}_{ n \ge 0}$ converges to $ \left\{\sup \limits_{m \ge 0} \ \norm{\thetavb(m)} < \infty \right\}$ as $n \uparrow \infty.$ Hence, $$\lim \limits_{n \uparrow \infty} \Prob{\sup \limits_{0 \le m \le n} \ \norm{\thetavb(m)} < \infty } = \Prob{\sup \limits_{m \ge 0} \ \norm{\thetavb(m)} < \infty} = 1.$$

\end{proof}

\begin{lem} \label{lem:vbvar}
    For all $k \ge 1$, $\sup \limits_{n \ge 0} \ \E{\Mv(n)^k} \le B(k) < \infty.$ Hence, $\sup \limits_{n \ge 0} \ \E{(\sup \limits_{0 \le m \le n}\norm{\thetavb(m)})^k} \le B(k)$ and $\E{(\sup \limits_{n \ge 0}\norm{\thetavb(n)})^k} \le B(k)$. The bound $B(k)$ is dependent on $k$ alone.
\end{lem}
\begin{proof}

\vspace{.2cm}

It is sufficient to prove the statement for integer valued $k.$ Since, $\E{\Mv(n)^p} \le 1 + \E{\Mv(n)^k}$ for $k-1 < p < k$ and $k \ge 1$. Let us fix an arbitrary integer $k.$ First, we observe that $\{\Mv(n)\}_{n \ge 0}$ is a sequence of positive random variables, hence the $k^{th}$ moment can be rewritten as follows:
\begin{equation}
\E{\Mv(n)^k} = k \int \limits_0 ^\infty x^{k-1} \Prob{\Mv(n) \ge x} \ dx.
\end{equation}
Using the Hoeffding-Azuma inequality of \eqref{eq:HA},
\begin{equation} \label{eq:momentbound}
    \int \limits_0 ^\infty x^{k-1} P(\Mv(n) \ge x) \ dx \le \int \limits_0 ^\infty x^{k-1} \exp{\left( \frac{-x^2}{\lambda^2 \sum \limits_{m \ge 0} a(m)^2} \right)} \ dx .
\end{equation}
We observe that the right hand side of \eqref{eq:momentbound} is only dependent on the moment, $k$, being calculated. Further, we know that it is integrable. Hence we let $B(k) \equiv k \int \limits_0 ^\infty x^{k-1} \exp{\left( \frac{-x^2}{\lambda^2 \sum \limits_{m \ge 0} a(m)^2} \right)} \ dx$. Now, from the construction of the sub-martingale sequence, we get that $\Mv(n) \ge \sup \limits_{0 \le m \le n} \ \norm{\thetavb(m)} \ a.s.$ Hence, $\sup \limits_{n \ge 0} \ \E{(\sup \limits_{0 \le m \le n}\norm{\thetavb(m)})^k} \le B(k).$

It is now left to show that $\E{\sup \limits_{m \ge 0}\norm{\thetavb(m)}^k} \le B(k).$ We know that $\lim \limits_{n \uparrow \infty} (\sup \limits_{0 \le m \le n}\norm{\thetavb(m)})^k = (\sup \limits_{m \ge 0} \norm{\thetavb(m)})^k \ a.s.$ Using the Monotone Convergence Theorem we get
\begin{equation}
    \lim \limits_{n \uparrow \infty} \ \E{(\sup \limits_{0 \le m \le n}\norm{\thetavb(m)}) ^k} = \E{(\sup \limits_{m \ge 0} \norm{\thetavb(m)}) ^k} \le B(k).
\end{equation}

\end{proof}

We have thus shown that the output layer weights $\thetavb$ are numerically stable when they are updated using clipped gradients. This property is independent of the clipping methodology used. It is, in particular, only required that the output layer be updated at every step using bounded update-values. In the above analysis, the updates are bounded by $\lambda.$ In addition, we have also shown that the norm of $\thetavb$ is such that every moment of it is bounded over time. Further, this bound is independent of time. In particular, in Lemma~\ref{lem:vbvar} we showed that the output layer weights are updated such that their norm is always bounded in variance. An algorithm has a tendency to be unstable when the weights have high variance over time. In Lemma~\ref{lem:vbvar} we showed that the output layer weights are every-moment-bounded, hence the algorithm is numerically stable. Viewed differently, numerical instability is a consequence of unpredictable gradient magnitudes. Since we clip the gradients before updating the output layer, we control the magnitude and ensure stability, we showed this in Lemma~\ref{lem:vbstable}. 

\subsection{The input layer behavior}

Traditionally, when the output layer is updated using clipped gradients, so are the other layers. In this paper, $\thetawu$ represents the vector of weights from the input layer. When updated using clipped gradients, we have:
\begin{equation}
    \label{eq:wuclip}
    \thetawu(n+1) = \thetawu(n) - a(n) \frac{\gradwu{\thetawu(n), x(n), y(n)}}{\frac{\norm{\gradwu{\thetawu(n), x(n), y(n)}}}{\lambda} \vee 1}
\end{equation}
It is expected that statements analogous to Lemmas~\ref{lem:vbstable} and \ref{lem:vbvar} hold true. On the face of it, clipping seems to be necessary to curtail the possible exploding gradients issue. Having said that, the hidden and input layers are more susceptible to the vanishing gradients problem. We believe that clipping gradients may amplify this issue. Further, in this section, we show that clipping gradients is unnecessary for the input and hidden layer updates. In particular, we show - when the output layer is numerically stable, so are the other layers, provided we use twice (or more) continuously differentiable squashing activations. Hence, the input layer parameter $\thetawu$ is updated as follows:
\begin{equation} \label{eq:wu}
    \thetawu(n+1) = \thetawu(n) - a(n) \gradwu{\thetawu(n), x(n), y(n)}
\end{equation}

\begin{lem}
    \label{lem:wustable}
    If $\sup \limits_{m \ge 0} \norm{\thetavb(m)} < \infty \ a.s.,$ then $\sup \limits_{m \ge 0} \norm{\thetawu(m)} < \infty \ a.s.$
\end{lem}
\begin{proof}

    \vspace{.2cm}

Let us start by defining an appropriate sub-martingale sequence: $\Mw(0) \equiv 0,$ $\Mw(n) \equiv \lVert \theta^{\W \ub}(0) \rVert + \sum \limits_{m=0}^{n-1} a(k) K \left(1 + \lVert \theta^{\vb b}(m)\rVert^2 \right),$ and the associated natural filtration $\Fw(0) \equiv \{\Phi, \Omega\}$ and $\Fw(n) \equiv \sigma \langle \theta(m), m < n \rangle.$ We first note that $\{\Mw(n)\}_{n \ge 0}$ is an almost surely increasing sequence of positive-valued random variables, and that $\Mw(n+1) - \Mw(n) = \lvert \Mw(n+1) - \Mw(n) \rvert = a(n) K(1 + \lVert \theta^{\vb b}(n)\rVert^2)$ for $n \ge 0.$ Also, from Lemma~\ref{lem:wubound}, we can conclude that $\Mw(n) \ge \sup \limits_{0 \le m \le n}\lVert \theta^{\W \ub}(m) \rVert \ a.s.$ The event $\left\{ \sup \limits_{m \ge 0} \norm{\thetavb(m)} < \infty \right\} = \underset{C \uparrow \infty}{\bigcup} \left\{ \sup \limits_{m \ge 0} \norm{\thetavb(m)} \le C \right\}.$ For $0 < c_1 \le c_2 < \infty,$ $\left\{ \sup \limits_{m \ge 0} \norm{\thetavb(m)} \le c_1 \right\} \subseteq \left\{ \sup \limits_{m \ge 0} \norm{\thetavb(m)} \le c_2 \right\}$, hence $\mathds{1}\left(  \sup \limits_{m \ge 0} \norm{\thetavb(m)} \le c_1  \right) \le \mathds{1} \left(  \sup \limits_{m \ge 0} \norm{\thetavb(m)} \le c_2 \right) \ a.s.,$ where $\mathds{1}(\cdotp)$ represents the indicator random variable. Hence, we can use Monotone Convergence Theorem to conclude that $\lim \limits_{C \uparrow \infty} \Prob{\set{\sup \limits_{m \ge 0} \norm{\thetavb(m)} \le C} } = \Prob{\set{\sup \limits_{m \ge 0} \norm{\thetavb(m) }< \infty} } = 1.$

Fix an arbitrary $0 < C < \infty,$ then we can use the following conditional version of the Hoeffding-Azuma Inequality, where we condition on the event $\set{\sup \limits_{m \ge 0} \norm{\thetavb(m)} \le C}$:

\begin{equation} 
    \begin{split}
      \Prob{\abs{\Mw(n) - \Mw(0)} \ge C^3 \left\lvert \sup \limits_{m \ge 0} \norm{\thetavb(m)} \right. \le C} &\le \exp \left(\frac{-C^6}{\sum \limits_{m=0}^{n-1} a(m)^2 K^2(1 + \lVert \theta^{\vb b}(n)\rVert^2)^2}\right) \\ &\le \exp \left(\frac{-C^6}{\sum \limits_{m=0}^{\infty} a(m)^2 K^2(1 + C^2)^2}\right).  
    \end{split}
\end{equation}
Since $\Mw(0) = 0,$ and $\Mw(n)$ is positive \textit{a.s.} for all $n,$ we have
\begin{equation} \label{eq:martingale1}
    \begin{split}
      \Prob{\Mw(n) \ge C^3  \left\lvert \sup \limits_{m \ge 0} \norm{\thetavb(m)} \right. \le C} \le \exp \left(\frac{-C^6}{\sum \limits_{m=0}^{\infty} a(m)^2 K^2(1 + C^2)^2}\right).  
    \end{split}
\end{equation}
If we let $C \uparrow \infty$ in \eqref{eq:martingale1}, then the left-hand side becomes

\begin{equation} \label{eq:martingale2}
 \lim \limits_{C \uparrow \infty} \Prob{ \Mw(n) \ge C^3 \left\lvert \sup \limits_{m \ge 0} \norm{\thetavb(m)} \right. \le C} = \Prob{\Mw(n) = \infty \left\lvert \sup \limits_{m \ge 0} \norm{\thetavb(m)} \right. < \infty},
\end{equation}
the right-hand side is such that
\begin{equation} \label{eq:martingale3}
  \lim \limits_{C \uparrow \infty} \ \exp \left(\frac{-C^6}{\sum \limits_{m=0}^{\infty} a(m)^2 K^2(1 + C^2)^2}\right) = 0.
\end{equation}
Putting equations \eqref{eq:martingale2} and \eqref{eq:martingale3} together we get:
\begin{equation}
  \Prob{ \Mw(n) = \infty \left\lvert \sup \limits_{m \ge 0} \norm{\thetavb(m)} \right. < \infty} = 0.
\end{equation}
If we consider the complementary event, and utilize that $\M(n) \ge \sup \limits_{0 \le m \le n}\lVert \theta^{\W \ub}(m) \rVert \ a.s.$ and that $\Prob{\sup \limits_{m \ge 0} \norm{\thetavb(m)}  < \infty} = 1,$ we get
$
  \Prob{ \sup \limits_{0 \le m \le n}\lVert \theta^{\W \ub}(m) \rVert < \infty} = 1, \ \forall \ n \ge 0.
$
Now, we note that almost surely $\ind{ \sup \limits_{0 \le m \le n}\lVert \theta^{\W \ub}(m) \rVert < \infty} \big{\downarrow} \ind{ \sup \limits_{m \ge 0}\lVert \theta^{\W \ub}(m) \rVert < \infty }$ as $n \uparrow \infty.$ Finally, we use the Dominated Convergence Theorem to conclude that $\Prob{ \sup \limits_{m \ge 0} \ \lVert \theta^{\W \ub}(m) \rVert < \infty} = 1.$

\end{proof}

In the above lemma, we showed that the input layer weights are automatically stable when the output layer weights are updated in a stable manner. Although we considered a shallow network in our analysis, a similar stability analysis will go through for deep neural networks (DNNs). For DNNs, we can show the following: the `current' layer weights are stable provided all the `successive' layer weights are updated in a stable manner. Again, once the output layer is stabilized, the remaining layers can be shown to be automatically stable through induction, starting from the output layer and moving back to the input layer, one layer at a time.

To prove Lemma~\ref{lem:wustable}, we did not use Lemma~\ref{lem:vbvar}. We only used the stability of the output layer. Suppose the statement of Lemma~\ref{lem:vbvar} holds, then we show, in the following Lemma, that the input layer weights are also every-moment bounded. However, unlike the bound we obtained before, for the input layer we get a bound that depends on $n$.

\begin{lem}
\label{lem:wuvar}
If $\E{(\sup \limits_{0 \le m \le n} \norm{\thetavb(m)})^k} < B(k,n),$ then $\E{(\sup \limits_{0 \le m \le n} \norm{\thetawu(m)})^k} < B'(k,n),$ where $k \ge 1,$ $B(k,n) < \infty$ and $B'(k,n)< \infty.$
\end{lem}

\begin{proof}

    \vspace{.2cm}

Recall the previously defined sub-martingale: $\Mw(0) \equiv 0$ and $\Mw(n) \equiv \lVert \theta^{\W \ub}(0) \rVert + \sum \limits_{m=0}^{n-1} a(m) K \left(1 + \lVert \theta^{\vb b}(m)\rVert^2 \right),$ and the natural filtration, $\Fw(0) \equiv \{\Phi, \Omega\}$ and $\Fw(n) \equiv \sigma \left\langle \theta(m), m < n \right\rangle,$ for $n \ge 1.$ Let us fix arbitrary integers $k,n \ge 1.$ Since, $k \ge 1,$ one can show that
\begin{equation}
\label{eq:wuvar1}
\E{(\Mw(n))^k} \le n^{k-1} \E{ \lVert \theta^{\W \ub}(0) \rVert^k + \sum \limits_{m=0}^{n-1} a(m)^k K^k \left(1 + \lVert \theta^{\vb b}(m)\rVert^2\right)^k}.
\end{equation}
Now, we use Assumption~\ref{asmp:initial} and $\E{(\sup \limits_{0 \le m \le n} \norm{\thetavb(m)})^{2k}} < B(2k,n)$ to bound the RHS of \eqref{eq:wuvar1} by some constant $B'(k,n) < \infty.$ The statement of the Lemma directly follows from the definition of the sub-martingale - $\Mw(n) \ge \sup \limits_{0 \le m \le n} \norm{\thetawu}(m) \ a.s.$
    
\end{proof}

As in the case of Lemma~\ref{lem:wustable}, we can modify Lemma~\ref{lem:wuvar} to account for DNNs as well. We can show the following: at any time-step suppose the weight vectors associated with each successive layer are every-moment bounded, then the weight vector of the current layer is also every-moment bounded.

When proving Lemma~\ref{lem:wuvar} we also showed that the sub-martingale sequence is square integrable, \textit{i.e.}, $\E{\Mw(m) ^2}< \infty$ for every $m \ge 0.$ For the square integrable sub-martingale, we can define the following quadratic variation process:
 $\angles{\Mw(n)} \equiv \sum \limits_{m=0}^n a(m)^2 K^2 \left(1 + \lVert \theta^{\vb b}(m)\rVert^2 \right)^2,$ $n \ge 2.$ When the output layer is stable, \textit{i.e.,} $\sup \limits_{n \ge 0} \norm{\thetavb(n)} < \infty \ a.s.,$ we get that $\lim \limits_{n \to \infty} \angles{\Mw(n)} = \sum \limits_{m=0}^n a(m)^2 K^2 \left(1 + \lVert \theta^{\vb b}(m)\rVert^2 \right)^2 < \infty \ a.s.$ (we rely here on the square summability of the step-size sequence as well). We know that $\lim \limits_{n \to \infty} \Mw(n)$ exists whenever $\lim \limits_{n \to \infty} \angles{\Mw(n)} < \infty \ a.s.$ Since our sub-martingale is an increasing sequence of positive random variables, we get that $\sup \limits_{n \ge 0} \Mw(n) < \infty \ a.s.$ and that $\sup \limits_{n \ge 0} \norm{\thetawu(n)} < \infty \ a.s.$

\begin{thm} \label{thm:stability}
 Under Assumptions~\ref{asmp:nnact}-\ref{asmp:steps}, the algorithms \eqref{eq:regressionupdate}, \eqref{eq:classificationupdate} and \eqref{eq:dqnupdate} are stable, \textit{i.e.,} $\sup \limits_{n \ge 0} \norm{\nnw(n)} < \infty \ a.s$. Further,
 $\E{(\sup \limits_{n \ge 0}\norm{\thetavb(n)})^k} \le B(k)$ for some $B(k) < \infty$ that is dependent on $k \ge 1$ alone, and $\E{(\sup \limits_{0 \le m \le n} \norm{\thetawu(m)})^k} < B'(k,n),$ where $k \ge 1$ and $B'(k,n)< \infty.$
\end{thm}

\begin{proof}

\vspace{.2mm}

In Lemma~\ref{lem:vbstable} we showed that $\sup \limits_{n \ge 0} \norm{\thetavb(n)} < \infty$ a.s. Then, in Lemma~\ref{lem:wustable} we showed that $\sup \limits_{n \ge 0} \norm{\thetawu(n)} < \infty$ a.s. whenever $\sup \limits_{n \ge 0} \norm{\thetavb(n)} < \infty$ a.s. Putting them together gives us the required stability result - $\sup \limits_{n \ge 0} \norm{\nnw(n)} < \infty \ a.s$.

The results with respect to the moments are proven in Lemmas~\ref{lem:vbvar} and \ref{lem:wuvar}.

\end{proof}

One key symptom of unreliable learning is the high variance encountered when plotting the learning curve, whose variance is a direct function of the variance of the neural network weight updates. In the above theorem, we show that learning is reliable and is independent of stochastic quantities such as the random seed. It does not, however, suggest good performance. Additional factors such as the size of the neural network and learning rate play an important role in ensuring good performance. Low variance is still key to good performance. \textit{The theorem statement can be used to conclude that every moment is bounded, indicating that performance is consistent, a property that is lacking in current deep learning methods.}

\section{Convergence Analysis} \label{sec:convergence}
In this section, we present the convergence analysis of the regression, classification and Deep Q-Learning algorithms. Generally speaking, we will utilize the tools from the theory of stochastic approximation algorithms for analyses, see \cite{borkar2008stochastic, harold1997stochastic}. Analyses for the supervised learning algorithms - regression and classification - are very similar. In order to avoid redundancy, we will provide a detailed analysis for regression, then provide the points of deviation for classification. Also, the analysis is based on the theory developed in \cite{benaim2005stochastic}. For Deep Q-Learning, we will base our convergence analysis on theory developed in \cite{borkar2006stochastic} and \cite{ramaswamy2021deep}.
\subsection{Regression} \label{sec:regressionconvergence}
We first rewrite \eqref{eq:regressionupdate} in order to apply the theory developed in \cite{benaim2005stochastic}.
\begin{equation}
    \label{eq:regressionupdaterewrite}
    \begin{split}
       \thetavb(n+1) &= \thetavb(n) - a(n) \left( \Exp{\frac{\gradvb{\nnw(n), \x, y}}{\nicefrac{\norm{\gradvb{\nnw(n), \x, y}}}{\lambda} \vee 1}} + \Mvb(n+1) \right),\\ 
       \thetawu(n+1) &= \thetawu(n) - a(n) \left(\Exp{\gradwu{\nnw(n), \x, y}} + \Mwu(n+1)\right),
    \end{split}   
\end{equation}

where $\mu$ is the data distribution, \textit{i.e.}, the probability distribution that generated the dataset used to train the regression network; $\Mvb(n+1) \equiv 
\frac{\gradvb{\nnw(n), \x, y}}{\nicefrac{\norm{\gradvb{\nnw(n), \x, y}}}{\lambda} \vee 1} - \Exp{\frac{\gradvb{\nnw(n), \x, y}}{\nicefrac{\norm{\gradvb{\nnw(n), \x, y}}}{\lambda} \vee 1}}$ and $\Mwu(n+1) \equiv 
\gradwu{\nnw(n), \x, y} - \Exp{\gradwu{\nnw(n), \x, y}}.$ Note also that we have omitted the $``r''$ subscript from the loss function. We do this to reduce clutter and because we want to use large parts of the analysis for the classification setting as well. In \cite{benaim2005stochastic}, algorithms such as \eqref{eq:regressionupdaterewrite} were studied for the general case where the objective function is a point-to-set map. Our objective function - $\nnw \mapsto \begin{pmatrix}\Exp{\frac{\gradvb{\nnw, \x, y}}{\nicefrac{\norm{\gradvb{\nnw, \x, y}}}{\lambda} \vee 1}} \\
\Exp{\gradwu{\nnw, \x, y}}
\end{pmatrix}$ - is trivially a point-to-set map where the function value is always the singleton $\set{\begin{pmatrix}\Exp{\frac{\gradvb{\nnw, \x, y}}{\nicefrac{\norm{\gradvb{\nnw, \x, y}}}{\lambda} \vee 1}} \\
\Exp{\gradwu{\nnw, \x, y}}
\end{pmatrix}}$. In order to apply the theory from \cite{benaim2005stochastic}, we need to show the following:

\vspace{.1cm}

\begin{fact} \label{factobj}
Almost surely, $\nnw \mapsto \set{\begin{pmatrix} \Exp{\frac{\gradvb{\nnw, \x, y}}{\nicefrac{\norm{\gradvb{\nnw, \x, y}}}{\lambda} \vee 1}} \\
\Exp{\gradwu{\nnw, \x, y}}
\end{pmatrix}}$ is a Marchaud map for $\nnw \in \overline{B}_{K'}(\bm{0}),$ where $K' \equiv \sup \limits_{n \ge 0} \norm{\nnw(n)}$ is a sample path dependent finite real number, and $\overline{B}_{K'}(\bm{0})$ represents the sphere of radius $K'$ centered at the origin.
\end{fact}

\vspace{.1cm}

\begin{fact} \label{factvb}
Almost surely, $\norm{\Exp{\frac{\gradvb{\nnw(n), \x, y}}{\nicefrac{\norm{\gradvb{\nnw(n), \x, y}}}{\lambda} \vee 1}}} \le KK'(1 + \norm{\nnw(n)})$ for all $n \ge 0,$ where $K' < \infty$ is a sample path dependent constant, further, $K$ is from Lemma~\ref{lem:wubound}.
\end{fact}

\vspace{.1cm}

\begin{fact} \label{factwu}
Almost surely, $\norm{\Exp{\gradwu{\nnw(n), \x, y}}} \le KK'(1 + \norm{\nnw(n)})$ for all $n \ge 0,$ where $K' < \infty$ is a sample path dependent constant, and $K$ is from Lemma~\ref{lem:wubound}.
\end{fact}

\vspace{.1cm}

\begin{fact} \label{factMartingale}
$\norm{\Mvb(n+1)} \le 2KK'(1 + \norm{\nnw(n)}) \ a.s.$ and $\norm{\Mwu(n+1)} \le 2KK'(1 + \norm{\nnw(n)}) \ a.s.$ The constants $K$ and $K'$ are as in the previous two facts.
\end{fact}

\vspace{.1cm}

Assuming these facts for now, we may use the theory developed in \cite{benaim2005stochastic} to conclude that \eqref{eq:regressionupdaterewrite}, and hence \eqref{eq:regressionupdate}, has the same asymptotic behavior as the associated differential inclusion $\dot{\nnw}(t) \in \set{\begin{pmatrix}\Exp{\frac{\gradvb{\nnw(t), \x, y}}{\nicefrac{\norm{\gradvb{\nnw(t), \x, y}}}{\lambda} \vee 1}} \\
\Exp{\gradwu{\nnw(t), \x, y}}
\end{pmatrix}}$, which is really the ordinary differential equation \begin{equation}
    \label{eq:regressionode}
    \dot{\nnw}(t) = \begin{pmatrix}\Exp{\frac{\gradvb{\nnw(t), \x, y}}{\nicefrac{\norm{\gradvb{\nnw(n), \x, y}}}{\lambda} \vee 1}} \\
\Exp{\gradwu{\nnw(t), \x, y}}
\end{pmatrix}.
\end{equation}

Now, consider Theorem~2 from Chapter~6 of \cite{aubin2012differential}. It states that

\vspace{.2cm}

\noindent
\textit{``Let $F$ be a continuous map from a closed subset $\mathcal{K} \subset \mathbb{R}^d$ to $\mathbb{R}^d.$ Let $x(\cdotp)$ be a solution trajectory to the o.d.e. $\dot{x}(t) = F(x(t))$ such that $x(t) \in \mathcal{K}$ for $t \ge 0.$ If the solution converges to some $x^* \in \mathcal{K}$, then $x^*$ is an equilibrium of $F.$''}

\vspace{.2cm}

It follows from Theorem~\ref{thm:stability} that there exists a sample path dependent compact set $\mathcal{C} \subset \mathbb{R}^d$ such that the algorithm iterates and the o.d.e. solution, \eqref{eq:regressionode}, tracked by it remain inside $\mathcal{C}.$ Suppose the algorithm \eqref{eq:regressionupdate} converges to $\theta(\infty),$ then so does the tracking o.d.e. solution. Now, we utilize the above stated theorem from viability theory (Theorem~2 from Chapter~6 of \cite{aubin2012differential}) to conclude that $ \begin{pmatrix}\Exp{\frac{\gradvb{\nnw(\infty), \x, y}}{\nicefrac{\norm{\gradvb{\nnw(\infty), \x, y}}}{\lambda} \vee 1}} \\
\Exp{\gradwu{\nnw(\infty), \x, y}} \end{pmatrix} = \vec{0},$ where $\vec{0}$ is the vector of all zeroes - $\begin{pmatrix}
    0 \\ \vdots \\ 0
\end{pmatrix}.$
We see that the regression update \eqref{eq:regressionupdate} converges to a set of neural network weights with the zero vector as the expected loss-gradient. 

It is left to show Facts~\ref{factobj}-\ref{factMartingale}, and we begin with the first in the list. In order to show the Marchaudness of the said set-valued map, we just need to show that the original point-to-point map is continuous. This is because the linear growth property readily follows from Lemma~\ref{lem:wubound}, and the compactness and convexity of the range follows from the fact that the range consists of singleton sets that are trivially compat and convex. It is also worth mentioning that the linear growth property can be proven using the arguments involved in showing Facts~\ref{factvb} and \ref{factwu}. In other words, we need to show that $\begin{pmatrix}\Exp{\frac{\gradvb{\nnw(n), \x, y}}{\nicefrac{\norm{\gradvb{\nnw(n), \x, y}}}{\lambda} \vee 1}} \\
\Exp{\gradwu{\nnw(n), \x, y}} \end{pmatrix} \to \begin{pmatrix}\Exp{\frac{\gradvb{\nnw(\infty), \x, y}}{\nicefrac{\norm{\gradvb{\nnw(\infty), \x, y}}}{\lambda} \vee 1}} \\
\Exp{\gradwu{\nnw(\infty), \x, y}} \end{pmatrix}$ when $\theta(n) \to \theta(\infty).$ First, let us define $g(\nnw, \x, y) \equiv \begin{pmatrix}\frac{\gradvb{\nnw, \x, y}}{\nicefrac{\norm{\gradvb{\nnw, \x, y}}}{\lambda} \vee 1} \\
\gradwu{\nnw, \x, y} \end{pmatrix}$ and $G(\nnw) \equiv \begin{pmatrix}\Exp{\frac{\gradvb{\nnw, \x, y}}{\nicefrac{\norm{\gradvb{\nnw, \x, y}}}{\lambda} \vee 1}} \\
\Exp{\gradwu{\nnw, \x, y}} \end{pmatrix}.$ Next, we make two observations: (i) the NN is composed of activation functions that are continuous (see Assumption~\ref{asmp:nnact}), (ii) the max operator is continuous. We couple these observations with the components of the loss-gradient given by \eqref{eq:reggrad}, then we get that for every fixed $x \in \Xit$ and $y \in \Yit,$ $g(\nnw(n), \x, y) \to g(\nnw(\infty), \x, y).$ Let $\hat{K} \equiv \sup \limits_{n \ge 0} \norm{\nnw(n)} \vee 1,$ since Lemma~\ref{lem:wubound} can be used to infer that $\norm{g(\nnw, \x, y)} \le K(1 + \norm{\theta}^2),$ we get
\begin{equation}
    \label{eq:regdct}
    \norm{g(\nnw(n), \x, y)} \le K\left(1 + \left(\sup \limits_{n \ge 0} \norm{\nnw(n)} \right)^2 \right) \le K  (1 + \hat{K}^2).
\end{equation}

Here, $\hat{K} < \infty$ is a sample-path dependent constant. Hence, we can use the Dominated Convergence Theorem, \cite{durrett2019probability}, to conclude that $G(\nnw(n)) \to G(\nnw(\infty)).$

To show facts~\ref{factvb} and \ref{factwu}, we observe the following:

\begin{equation}
    \label{eq:regfact23}
    \norm{g(\nnw(n), \x, y)} \le K\left(1 + \norm{\nnw(n)} \times \sup \limits_{n \ge 0} \norm{\nnw(n)}  \right) \le K K' (1 + \norm{\nnw(n)}),
\end{equation}
where $K' \equiv \sup \limits_{n \ge 0} \norm{\nnw(n)} < \infty$ is a sample path dependent constant obtained from Theorem~\ref{thm:stability}. Finally, Fact~\ref{factMartingale} directly follows from the definitions of $\Mvb$ and $\Mwu$ in combination with Facts~\ref{factvb} and \ref{factwu}.

\subsection{Classification} \label{sec:classifyconvergence}
Like regression, we start by rewriting the cross-entropy based classification given by \eqref{eq:classificationupdate} as follows:

\begin{equation} \label{eq:classificationupdaterewrite}
    \nnw(n+1) = \nnw(n) - a(n) \left[\Exp{\gradg{\nnw(n), \x, y)}} + M(n+1) \right],
\end{equation}
where $M(n+1) \equiv \gradg{\nnw(n), \x(n), y(n))} - \Exp{\gradg{\nnw(n), \x, y)}}.$ Recall that we do not have to clip the gradients before updating the output layer since they are bounded to begin with. Like in Section~\ref{sec:regressionconvergence}, we state a few facts that facilitate the application of the theory developed in \cite{benaim2005stochastic}. However, unlike in the previous section, we will not show that these facts hold true for classification with cross-entropy loss. This is because the steps involved are identical to those in the previous section.

\vspace{.1cm}

\begin{fact}
    \label{factMarchaudc}
    Almost surely, $\nnw \mapsto \set{\Exp{\gradg{\nnw, \x, y)}}}$ is a Marchaud map for $\nnw \in \overline{B}_{K'}(\bm{0}),$ where $K' \equiv \sup \limits_{n \ge 0} \norm{\nnw(n)}$ is a sample path dependent finite real number, and $\overline{B}_{K'}(\bm{0})$ represents the sphere of radius $K'$ centered at the origin.
\end{fact}

\vspace{.1cm}

\begin{fact} \label{factvbc}
Almost surely, $\norm{\Exp{\gradg{\nnw(n), \x, y)}}} \le KK'(1 + \norm{\nnw(n)})$ for all $n \ge 0,$ where $K' < \infty$ is a sample path dependent constant from the previous fact, and $K$ is from Lemma~\ref{lem:wubound2}.
\end{fact}

\vspace{.1cm}

\begin{fact} \label{factMartingalec}
$\norm{M(n+1)} \le 2KK'(1 + \norm{\nnw(n)}) \ a.s.$ The constants $K$ and $K'$ are as in the previous two facts.
\end{fact}

\vspace{.1cm}

The classification algorithm given by \eqref{eq:classificationupdate} or \eqref{eq:classificationupdaterewrite} tracks a solution to the ordinary differential equation given by $\dot{\nnw}(t) = \Exp{\gradg{\nnw(t), \x, y)}}.$ Let $\nnw(\infty)$ be the limit of the algorithm, then it is the equilibrium of the function $\Exp{\gradg{\cdotp, \x, y)}},$ \textit{i.e.,} $\Exp{\gradg{\nnw(\infty), \x, y)}} = \vec{0}.$

\subsection{Deep Q-Learning} \label{sec:dqnconvergence}
The convergence of Deep Q-Learning has been analyzed in \cite{ramaswamy2021deep} using the theory of dynamical systems. The ideas were built on the theory developed in \cite{borkar2006stochastic}. The analysis in \cite{ramaswamy2021deep} makes identical assumptions with regards to the activations, step-sizes, and state and action spaces, as this paper. However, there are two major deviations: (1) they additionally assume stability of the iterates, (2) they do not consider that the output layer weights are updated using clipped gradients. In the previous section - Section~\ref{sec:algo} - we have shown that the NN weights can be updated in a stable manner when the output layer weights are updated using some gradient clipping technique which ensures that the update-value at every step is norm-bounded. 

With respect to the convergence analysis of \eqref{eq:dqnupdate}, the key properties of the clipped Bellman loss gradient are similar to the traditional loss gradient considered in \cite{ramaswamy2021deep}. In particular, the continuity of the clipped loss gradient and the bound to the update-value (clipped gradient for the output layer, usual gradient for other layers) as a function of the NN weights are the pertinent ones. Hence, we can expect something similar to the main result, Theorem 1, of \cite{ramaswamy2021deep} if we were to follow the analysis presented there. The first step is to rewrite \eqref{eq:dqnupdate} in order to apply the analysis:

\begin{equation} \label{eq:dqnupdaterewrite}
    \begin{split}
        \thetavb(n+1) &= \thetavb(n) - a(n) \left[ \int_{\mathcal{S} \times \mathcal{A}} \frac{g_{vb} (\nnw(n), \x, \mathds{\alpha})}{\large\nicefrac{\norm{g_{vb}(\nnw(n), \x, \mathds{\alpha})}}{\lambda} \vee 1} \ \mu(n, d\x, d\alpha) + \Mvb(n+1) \right], \\
        \thetawu(n+1) &= \thetawu(n) - a(n) \left[\int_{\mathcal{S} \times \mathcal{A}} g_{wu}(\nnw(n), \x, \alpha) \ \mu(n, d\x, d\alpha) + \Mwu(n+1) \right],
    \end{split}
\end{equation}
where $\mu(n) \in \mathcal{P}(\mathcal{S} \times \mathcal{A})$ such that $\mu(n, \x(n,k), \alpha(n,k)) = 1/K$ for $1 \le k \le K$ and $\mu(n, x, \alpha) = 0$ for other $x \in \mathcal{S}$ and $\alpha \in \mathcal{A};$ \\ $g_{vb} (\nnw, \x, \mathds{\alpha}) = 2 \left( Q(\x,\alpha,\theta) - r(\x,\alpha) - \gamma \ \int_{\mathcal{S}}\max \limits_{a' \in \mathcal{A}} Q(z, a', \theta) \ p(dz | \x, \alpha, \nnw) \right) \nabla_{vb}Q(\x, \alpha, \nnw)$ such that $p$ is the state transition kernel; $g_{wu} (\nnw, \x, \mathds{\alpha}) = 2 \left( Q(\x,\alpha,\theta) - r(\x,\alpha)  - \gamma \ \int_{\mathcal{S}}\max \limits_{a' \in \mathcal{A}} Q(z, a', \theta) \ p(dz | \x, \alpha, \nnw) \right)$ $ \nabla_{wu}Q(\x, \alpha, \nnw);$ $$\Mvb(n+1) \equiv \frac{1}{K} \sum \limits_{k=1}^K \frac{\nabla_{vb} \ell_b(\nnw(n), \x(n,k), \mathds{\alpha}(n,k))}{\large\nicefrac{\norm{\nabla_{vb} \ell_b(\nnw(n), \x(n,k), \mathds{\alpha}(n,k))}}{\lambda} \vee 1} - \int_{\mathcal{S} \times \mathcal{A}} \frac{g_{vb} (\nnw(n), \x, \mathds{\alpha})}{\large\nicefrac{\norm{g_{vb}(\nnw(n), \x, \mathds{\alpha})}}{\lambda} \vee 1} \ \mu(n, d\x, d\alpha) ;$$
$$\Mwu(n+1) \equiv \frac{1}{K} \sum \limits_{k=1}^K \frac{\nabla_{wu} \ell_b(\nnw(n), \x(n,k), \mathds{\alpha}(n,k))}{\large\nicefrac{\norm{\nabla_{wu} \ell_b(\nnw(n), \x(n,k), \mathds{\alpha}(n,k))}}{\lambda} \vee 1} - \int_{\mathcal{S} \times \mathcal{A}} g_{wu} (\nnw(n), \x, \mathds{\alpha}) \ \mu(n, d\x, d\alpha) .$$

The modified Q-Learning algorithm given by \eqref{eq:dqnupdaterewrite} can be analyzed as in \cite{ramaswamy2021deep} to conclude that the limit $\nnw(\infty)$ satisfies the following property:
\begin{equation}
    \label{eq:dqn_conclude}
    \begin{split}
        \int_{\mathcal{S} \times \mathcal{A}} \frac{g_{vb} (\nnw(\infty), \x, \mathds{\alpha})}{\large\nicefrac{\norm{g_{vb}(\nnw(\infty), \x, \mathds{\alpha})}}{\lambda} \vee 1} \ \mu(\infty, d\x, d\alpha) &=  \vec{0}\\
        \int_{\mathcal{S} \times \mathcal{A}} g_{wu}(\nnw(\infty), \x, \alpha) \ \mu(\infty, d\x, d\alpha) &= \vec{0},
    \end{split}
\end{equation}

where $\mu(\infty)$ is a probability measure on the state-action space such that it is the limit of the $\mu(n)$ measures, as $n \to \infty$, in the Prohorov metric space. It must be noted that $\mu(\infty)$ is a function of the frequency with which state-action pairs are used in the training process.

Let us summarize the three convergence results in the form of a theorem.

\begin{thm}
    Under Assumptions~\ref{asmp:nnact}-\ref{asmp:steps}, we have that
    \begin{itemize}
        \item[(i)] The output-layer clipped regression algorithm \eqref{eq:regressionupdate} is stable and converges to $\nnw(\infty)$ such that $ \begin{pmatrix}\Exp{\frac{\nabla_{vb} \ell_r(\nnw(\infty), \x, y)}{\nicefrac{\norm{\nabla_{vb} \ell_r(\nnw(\infty), \x, y)}}{\lambda} \vee 1}} \\
\Exp{\nabla_{wu} \ell_r(\nnw(\infty), \x, y)} \end{pmatrix} = \vec{0},$ where $\mu$ is the underlying data distribution.
        \item[(ii)] The output-layer clipped classification algorithm \eqref{eq:classificationupdate} is stable and converges to $\nnw(\infty)$ such that $\Exp{\nabla_{\theta} \ell_c(\nnw(\infty), \x, y)\nnw(\infty), \x, y)} = \vec{0},$ where $\mu$ is as above.
        \item[(iii)] The output-layer clipped Q-Learning algorithm \eqref{eq:dqnupdate} is stable and converges to $\nnw(\infty)$ such that
        \begin{equation}
    \begin{split}
        \int_{\mathcal{S} \times \mathcal{A}} \frac{g_{vb} (\nnw(\infty), \x, \mathds{\alpha})}{\large\nicefrac{\norm{g_{vb}(\nnw(\infty), \x, \mathds{\alpha})}}{\lambda} \vee 1} \ \mu(\infty, d\x, d\alpha) &=  \vec{0}\\
        \int_{\mathcal{S} \times \mathcal{A}} g_{wu}(\nnw(\infty), \x, \alpha) \ \mu(\infty, d\x, d\alpha) &= \vec{0},
    \end{split}
\end{equation}
where $g_{vb} (\nnw, \x, \mathds{\alpha}) = 2 \left(  Q(\x,\alpha,\theta) - r(\x,\alpha) - \gamma \ \int_{\mathcal{S}}\max \limits_{a' \in \mathcal{A}} Q(z, a', \theta) \ p(dz | \x, \alpha, \nnw) \right) \nabla_{vb}Q(\x, \alpha, \nnw),$ $g_{wu} (\nnw, \x, \mathds{\alpha}) = 2 \left(  Q(\x,\alpha,\theta) - r(\x,\alpha) - \gamma \ \int_{\mathcal{S}}\max \limits_{a' \in \mathcal{A}} Q(z, a', \theta) \ p(dz | \x, \alpha, \nnw) \right)$ $ \nabla_{wu}Q(\x, \alpha, \nnw)$ and $\mu(\infty)$ is a limit of the $\mu(n)$ measures in the Prohorov metric.
    \end{itemize}
\end{thm}

\begin{proof}

    \vspace{.2cm}

We have shown stability in Theorem~\ref{thm:stability}. We analyzed the convergence of regression, classification and Q-Learning in Section~\ref{sec:regressionconvergence}, \ref{sec:classifyconvergence} and \ref{sec:dqnconvergence}, respectively.

\end{proof}

\section{Experimental Results} \label{sec:exp}
As stated earlier, we present a novel activation function called Truncated GELU (tGELU). It is obtained by modifying the Gaussian Error Linear Unit (GELU) activation function \cite{1}. It has two parameters $t_l \leq 0$ and $t_r\geq 0$ - the left and right threshold values. Let us use $\mathcal{N}$ to denote a normal random variable with zero mean and unit variance. Then, tGELU can be specified as follows in \eqref{tgelu}.

\begin{equation}\label{tgelu}
tGELU(x) \stackrel{\Delta}{=} 
\begin{cases}
\quad t_rP(\mathcal{N}\leq t_r) + (x-t_r)P(\mathcal{N}\geq x-t_r) &, \quad x\geq t_r \\
\quad xP(\mathcal{N}\leq x) &, \quad 0\leq x \leq t_r \\
\quad xP(\mathcal{N}\geq x) &, \quad t_l\leq x \leq 0 \\
\quad t_lP(\mathcal{N}\geq t_l) + (x-t_l)P(\mathcal{N}\leq x-t_l) &, \quad x\leq t_l \\
\end{cases}
\end{equation}

\begin{figure}[h]
\centering
\includegraphics[scale=0.5]{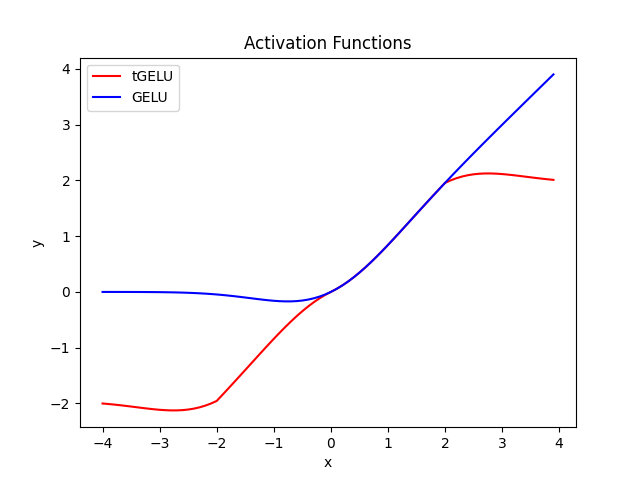}
\includegraphics[scale=0.5]{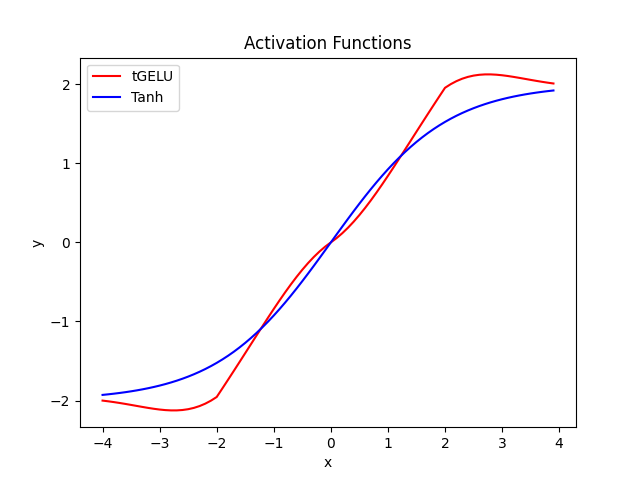}
\caption{The blue curves in both plots represent tGELU activation function with $t_l = -2$ and $t_r = 2.$ In the left-hand plot, tGELU is compared to the standard GELU. In the right-hand plot, tGELU is compared to $2Tanh(x/2)$.}
\label{fig:tgelu}
\end{figure}

Truncated GELU with $t_l=-2$ and $t_l=2$ is illustrated in Figure \ref{fig:tgelu}. Here, it is compared to the standard GELU function and the $x \mapsto 2Tanh(x/2)$ function. When compared to GELU, it does not discard all negative input, this may be desirable for some applications. It has a similar signature as compared to $2Tanh(x/2)$, but tGELU can be made asymmetric by choosing appropriate values for $t_l$ and $t_r,$ while $aTanh(x/a)$ is symmetric for every $a \ge 1$. Note that, by choosing $a > 1,$ $aTanh(x/a)$ saturated slower than $Tanh(x).$

In our experiments we compare the performances of DNNs with tGELU ($t_l = -1$ and $t_r = 1$) and GELU. We train the DNN with tGELU using our routine \eqref{eq:genericupdate}, while the DNN with GELU is trained using standard stochastic gradient descent (SGD). We consider one task from classification, that of classifying images as one of two types, cats or dogs. The other is the control task of balancing a cartpole using Deep Q-Learning. Let us begin with the classification task. We obtained the dataset containing images of cat and dog from Kaggle (https://www.kaggle.com/datasets/tongpython/cat-and-dog). It contains $10000$ images in total, with equal number of dog and cat images. For training, we use $8000$ images with equal number of dog and cat images. The rest of the $2000$ images are used for testing. Before using the images we perform several pre-processing steps on the images. We resize all the images to a dimension of $64 \times 64$ and convert them to gray-scale. Each pixel of an image now lies in the range $[0,256]$. Subsequently, we normalize the range of values of pixels to [-1,1]. 

We implement the Convolutional Neural Network (CNN) using the PyTorch Python library \footnote{ The code related to classification task could be found at this URL: \url{https://github.com/namansaxena9/tGELU/}}. The neural network architecture contains several pairs of convolutional and max-pooling layers. The convolutional layers are responsible for feature extraction and the max-pooling layers are responsible for aggregating these features and decreasing the image dimension further. The output of the last pair of convolutional and max-pooling layers is fed into a fully connected feed-forward neural network with one hidden layer. The last layer of the feed-forward network produces the probabilities of the image being a cat and a dog. We use binary cross-entropy loss to train the network. We use the Stochastic Gradient Descent (SGD) optimizer from PyTorch library. We perform experiments using both the GELU and tGELU activations. While the GELU-DNN is updated using SGD, tGELU-DNN is updated using gradient clipping, only for the output layer. Note that to train the GELU-DNN we do not use gradient clipping of any sort. 
The models are trained for $400$ epochs. In each epoch, we use the training dataset containing $8000$ images. We use a mini-batch size of $256$ images for every training step. Further, every $20$ epochs, we use the hold-out test data containing $2000$ images to obtain the accuracy score of the model. This constitutes the evaluation step.
We compare the performance obtained using tGELU and GELU activation functions in Figure \ref{fig:tgeluvsgelu}. We observe that test accuracy obtained with tGELU with gradient clipping is better than that obtained with GELU. Further, variance in the test accuracy is evidently less for tGELU with gradient clipping which goes on to confirm the conclusion of Lemmas \ref{lem:wuvar} and \ref{lem:vbvar}.

\begin{figure}
\centering
\includegraphics[scale=0.6]{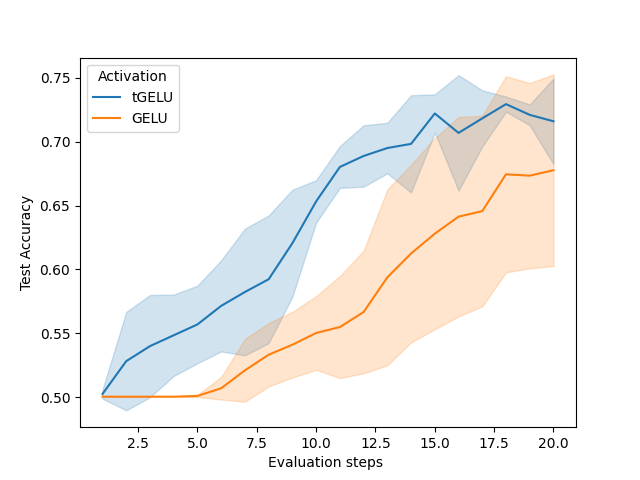}
\caption{Comparison of test accuracy of the cat-dog classification task for tGELU and GELU activation functions. Along x-axis we plot the training epochs, and the test accuracy is plotted along y-axis. The blue curve represents the test accuracy of the tGELU-DNN, and the red curve represents the accuracy of GELU-DNN.}
\label{fig:tgeluvsgelu}
\end{figure}

For the control task, we modified the CartPole-v1 environment of the OpenAI Gym library to include 3 actions: move-left, move-right, stay\footnote{ The code related to control task could be found at this URL: \url{https://github.com/namansaxena9/tGELU/}}. The agent has to maintain the pole angle between -12 to 12 degree. If the pole falls out of this range then the control objective has not been achieved and the episode terminates. The episode also terminates upon completing $1000$ time steps in the environment.
We trained an agent using Deep Q-learning on the cart-pole balancing task. We update the parameters of the network after each step and evaluate the performance of the agent after taking $5000$ environment steps. In each evaluation step we obtain the total reward of $5$ episodes and then use the average of the total reward across $5$ episodes as the measure of performance. The agent receives a reward of $+1$ when the cart-pole is in the $-12$ to $+12$ degree range. We use a discount factor of $0.99$ to calculate the Q-factors. The agent is trained for  $1$ million training steps using $\epsilon$-greedy policy. From $\epsilon$-greedy policy we mean, with probability $\epsilon$ random action is taken and with probability $1-\epsilon$ the action having the current highest Q-value is taken. At the start of the training the value of $\epsilon$ is take to be 1 and later the value is linearly annealed to 0.1 using initial 100k time steps in the environment. After 100k time steps the value of $\epsilon$ is kept constant at 0.1.  

\begin{figure}
    \centering
    \includegraphics[scale=0.6]{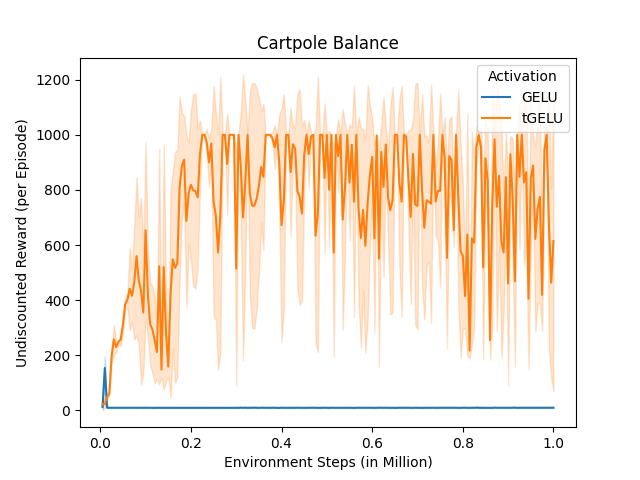}
    \caption{Comparison of the performance of Deep Q-Network with tGELU and GELU. Along x-axis, we plot the episodes and along y-axis we plot the total accumulated reward per episode. Note that our experiment contains $200$ performance evaluations (each performance evaluation takes place after an interval of 5000 environment steps), hence $200$ values. We rescale them so that the range of x-axis is between $0$ and $10^6.$}
    \label{fig:reward}
\end{figure}

A target network \cite{2} is usually used to stabilize the training of the agents in the off-policy setting. The target network is a copy of the main neural network. It is typically updated using the main network parameters at regular intervals using moving average scheme. First, we conducted an experiment to check our conjecture that our framework can be used to omit target networks. In this experiment, we trained the GELU and tGELU DQNs using Deep Q-Learning with the target network \textit{omitted} and noted the total reward performance. We again use the Stochastic Gradient Descent (SGD) optimizer of the PyTorch library. The tGELU-DQN is updated using \eqref{eq:dqnupdate}, and the GELU-DQN is updated using SGD. The results are compiled in Figure~\ref{fig:reward}. It illustrates that the tGELU-DQN - with gradient clipping for the output layer - outperforms the GELU-DQN by a large margin. In fact, GELU-DQN experiences finite-time explosion during every run of the experiment due to the absence of a target network. We also plotted the value of squared Bellman loss at every step of the experiment in Figure \ref{fig:bellman_loss}. It illustrates that the Bellman loss associated with the GELU-DQN explodes suddenly, following which the agent stops learning while for the agent using tGELU the Bellman loss curve shows a normal behavior. Hence the tGELU-DQN updated using \eqref{eq:dqnupdate} is stable even though there is no target network used.

\begin{figure}
    \centering
    \includegraphics[scale=0.7]{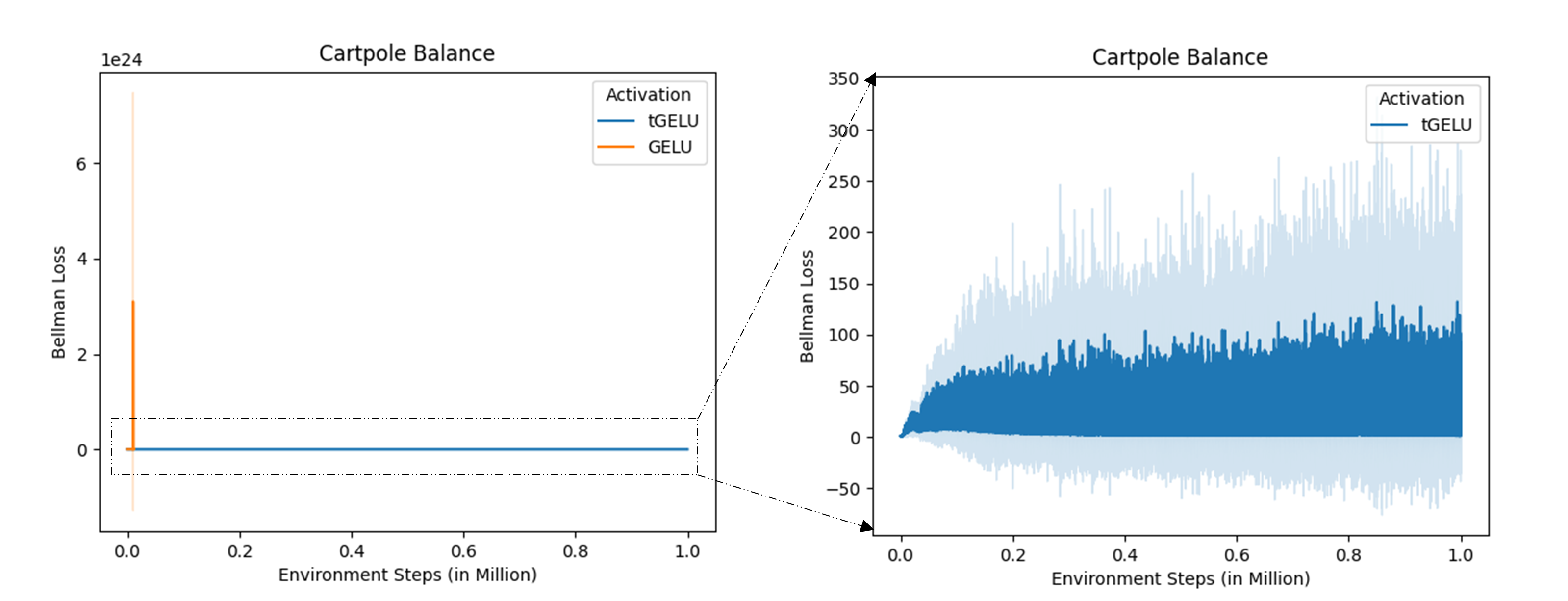}
    \caption{Along x-axis we plot the environment steps, along y-axis, we plot the per-step squared Bellman loss. The blue curve corresponds to tGELU-DQN and the orange curve corresponds to GELU-DQN. The figure to the right zooms onto the blue curve (tGELU performance) alone.}
    \label{fig:bellman_loss}
\end{figure}

For tGELU-DQN updated using \eqref{eq:dqnupdate}, we compared performance when using a target network, to the performance when omitting it. The results are illustrated in the left-plot of Figure~\ref{fig:target}. From this it is clear that inclusion of the target network makes marginal difference to the total reward obtained (see Figure \ref{fig:target}(left)). On the contrary, it is observed that without a target network agent with tGELU performs slightly better. The training of the agent using tGELU is already stable and hence adding target network does not make much of a difference.
We conducted a similar experiment for GELU-DQN \eqref{eq:dqnupdate}. The results from this experiment are illustrated in the right-plot of Figure~\ref{fig:target}. Clearly, using target network substantially improved the performance of the agent as compared to not using one(see Figure \ref{fig:target}(right)).

\begin{figure}
    \centering
    \includegraphics[scale=0.5]{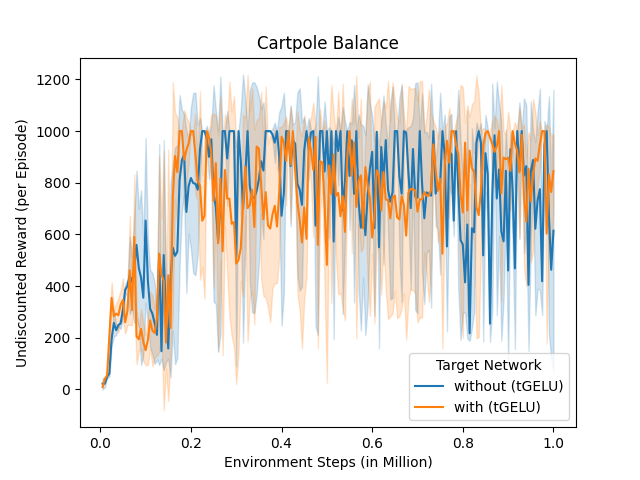}
    \includegraphics[scale=0.5]{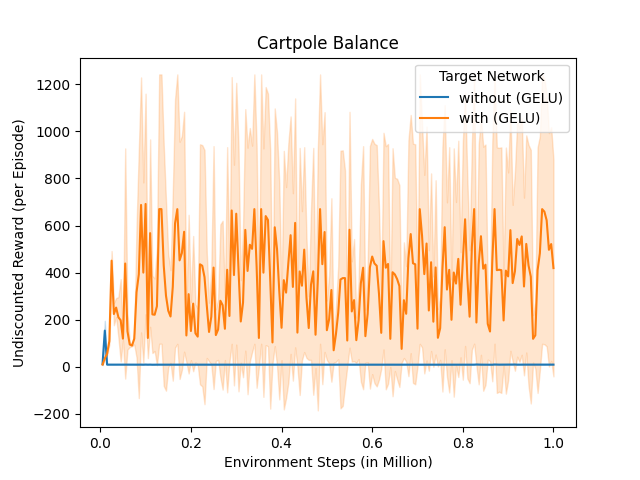}
    \caption{In the left plot, we compare the performance of tGELU-DQN with and without a target network. The blue plot represents the performance without a target network, while the orange plot represents the performance with one. It can be seen that the blue curve is above the orange curve for most of the experiment. The right plot illustrates a similar experiment for GELU-DQN. Here, the use of a target network greatly enhances performance.}
    \label{fig:target}
\end{figure}

Finally, we compare the performance of the tGELU-DQN that is trained using \eqref{eq:dqnupdate}, without a target network, to the performance of GELU-DQN that is trained using SGD while using a target network. The results are illustrated in Figure~\ref{fig:gelu_tgelu}. In Figure \ref{fig:gelu_tgelu} (left), the discounted sum of the rewards is plotted, while in Figure \ref{fig:gelu_tgelu} (right), the total reward is plotted. In the case of discounted reward per episode we individually obtain the discounted summation of reward for 5 episodes and then take the average across the 5 episodes. The same procedure is followed for total reward per episode with the difference that undiscounted summation is used. From Figure \ref{fig:gelu_tgelu}, we observe that for both discounted reward and total reward, the performance difference is marginal and tGELU without target network is slightly better.   

\begin{figure}
    \centering
    \includegraphics[scale=0.5]{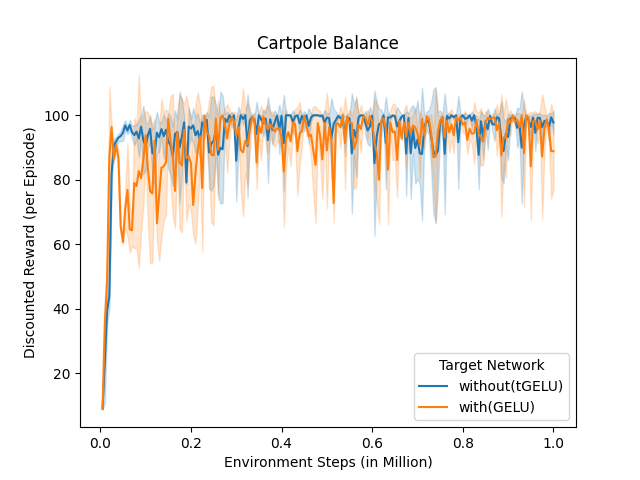}
    \includegraphics[scale=0.5]{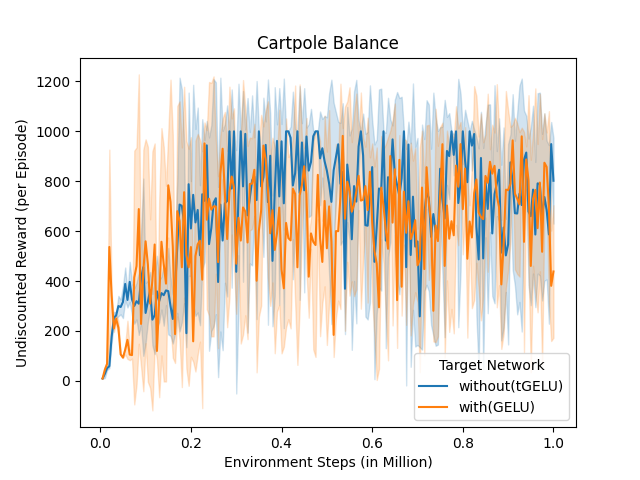}
    \caption{This figure compares the performance of tGELU-DQN updated using \eqref{eq:dqnupdate}, without a target network, and the performance of GELU-DQN updated using SGD with a target network. Along x-axis is plot the $200$ performance evaluations (each performance evaluation takes place after an interval of 5000 environment steps), rescaled so that the range is $0$ to $10^6$. Along the y-axis, we plot the discounted cumulative reward per episode to obtain the left-plot. We plot the total reward per episode along y-axis for the right-plot.}
    \label{fig:gelu_tgelu}
\end{figure}

\section{Conclusions}

In this paper we studied the problem of training DNNs using the stochastic gradient descent algorithm for supervised and unsupervised learning problems. We focused on training stability and performance variability. We analyzed DNNs that were only composed of squashing activations. To train them, we modified SGD so that only the output layer is updated using clipped gradients. The rest of the DNN (input and hidden layers) is updated using standard gradients. We showed that DNNs with squashing activations, trained this way, are numerically stable. In particular, we observed that the input and hidden layers can be stabilized by focusing on stability of the output layer alone. We achieved this by ensuring that the output layer is updated using bounded values - clipped gradients - at every timestep. One important consequence of stability, particularly for DQL is in eliminating the need for target networks.

Our framework leads to DNN updates such that their norms are ``every moment bounded'' for the entire duration of training. This reduced variance results in smooth learning and consistent performance associated with the final set of weights found. Through experiments, we showed that our framework is robust to parameters such as the random seed. Since our framework requires squashing activations, we developed tGELU, a new activation with very desirable properties. Unlike similar ones from literature, tGELU has an extended range and does not suffer from the vanishing gradient problem. Our experiments surrounding DQL suggest that DQN composed of tGELUs, trained using our routine, without a target network, perform better than classic DQL composed of GELUs, trained using a target network. Finally, we showed that the DNN must be initialized using probability distributions with compact supports, e.g., truncated Gaussian or Laplace distributions.

\section*{Acknowledgments.}
SB was supported by a J.C. Bose Fellowship, Project No.~DFTM/02/3125/M/04/AIR-04 from DRDO under DIA-RCOE, a project from DST-ICPS, and the RBCCPS, IISc.


\bibliographystyle{informs2014} 
\bibliography{reference} 


\end{document}